\documentclass[twoside,11pt]{article}

\usepackage[preprint]{jmlr2e}

\usepackage{hyperref}       
\usepackage{url}            
\usepackage{booktabs}       
\usepackage{amsfonts}       
\usepackage{nicefrac}       
\usepackage{microtype}      

\usepackage{amsmath, amssymb, enumerate, bm, capt-of,mathtools, accents}
\usepackage{hyperref}
\usepackage[capitalise]{cleveref}

\usepackage{relsize}
\usepackage{xcolor}
\usepackage{fp, tikz, pgfplots}
\usetikzlibrary{arrows,shapes,backgrounds,patterns,fadings,matrix,arrows,calc,
	intersections,decorations.markings,
	positioning,external,arrows.meta}
\usepgfplotslibrary{external}
\usepgfplotslibrary{fillbetween}
\usepgfplotslibrary{statistics}
\usepackage{etoolbox} 

\newtheorem{assumption}{Assumption}[section]

\makeatletter
\newcommand{\subalign}[1]{%
	\vcenter{%
		\Let@ \restore@math@cr \default@tag
		\baselineskip\fontdimen10 \scriptfont\tw@
		\advance\baselineskip\fontdimen12 \scriptfont\tw@
		\lineskip\thr@@\fontdimen8 \scriptfont\thr@@
		\lineskiplimit\lineskip
		\ialign{\hfil$\m@th\scriptstyle##$&$\m@th\scriptstyle{}##$\crcr
			#1\crcr
		}%
	}
}
\makeatother

\pgfplotsset{width=5\columnwidth /5, compat = 1.13, 
	height = 60\columnwidth /100, grid= major, 
	legend cell align = left, ticklabel style = {font=\scriptsize},
	every axis label/.append style={font=\small},
	legend style = {font=\tiny},title style={yshift=-7pt, font = \small} }

\tikzset{cross/.style={cross out, draw=black, minimum size=10*(#1-\pgflinewidth), inner sep=0pt, outer sep=0pt},cross/.default={1pt}}

\newcommand{\ubar}[1]{\underaccent{\bar}{#1}}

\jmlrheading{1}{2021}{1-48}{01/21}{10/00}{meila00a}{Armin Lederer, Jonas Umlauft and Sandra Hirche}

\ShortHeadings{Uniform Error Bounds for Gaussian Process Regression}{Lederer, Umlauft and Hirche}
\firstpageno{1}

\begin{document}

\title{Uniform Error and Posterior Variance Bounds for Gaussian Process Regression with Application to Safe Control}

\author{\name Armin Lederer \email armin.lederer@tum.de \\
		\name Jonas Umlauft \email jonas.umlauft@tum.de\\
		\name Sandra Hirche \email hirche@tum.de\\
       \addr Chair of Information-oriented Control\\
       Department of Electrical and Computer Engineering\\
       Technical University of Munich\\
       Munich, Germany}

\editor{Kevin Murphy and Bernhard Sch{\"o}lkopf}

\maketitle

\begin{abstract}
	In application areas where data generation is expensive, Gaussian processes 
	are a preferred supervised learning model due to their high 
	data-efficiency. Particularly in model-based control, Gaussian processes 
	allow the derivation of performance guarantees using probabilistic model 
	error bounds. To make these approaches applicable in practice, two open 
	challenges must be solved i) Existing error bounds rely on prior 
	knowledge, which might not be available for many real-world tasks. 
	(ii) The relationship between training data and the posterior variance, 
	which mainly drives the error bound, is not well understood and prevents 
	the asymptotic analysis. This article addresses these 
	issues by presenting a novel uniform error bound using Lipschitz 
	continuity and an analysis of the posterior variance function for a large 
	class of kernels. Additionally, we show how these results can be used to 
	guarantee safe control of an unknown dynamical system and provide numerical 
	illustration examples.
\end{abstract}


\section{Introduction}
Modeling nonlinear systems using supervised learning 
techniques~\citep{Norgard2000} enabled model-based control to succeed in highly 
complex tasks, such as controlling a robotic unicycle in 
simulation~\citep{Deisenroth2011a}. Nevertheless, the application of 
learning-based control in real-world safety-critical scenarios, like autonomous 
driving, assistive and rehabilitation robotics, or personalized disease treatment is rare, 
due to missing performance guarantees 
for the resulting closed-loop behavior. Empirical evaluations exist, 
e.g.,~\citep{Huval2015} for autonomous cars, but are insufficient if the 
physical integrity of systems is at risk. As in human-centric applications 
training data is expensive to acquire, Gaussian 
processes (GPs) are particularly appealing as the regression 
generalizes well for small training sets. Furthermore, GPs are a probabilistic 
method based on Bayesian principles~\citep{Rasmussen2006}, which allows to 
properly encode prior knowledge and to quantify the uncertainty in the 
model. 

Due to these properties, GPs gained increasing attention in the field of 
reinforcement learning and system identification. Especially, when safety
guarantees are necessary, GPs are favored in reinforcement 
learning~\citep{Berkenkamp,	Berkenkamp2016,Berkenkamp2017, Koller2018}
as well as control~\citep{Berkenkamp2015, Umlauft2017, Beckers2018, 
Lederer2020c, Umlauft2018, Helwa2019}.
These approaches heavily rely on error bounds of GP regression 
and are therefore limited by the strict assumptions made in previous works on 
GP uniform error bounds~\citep{Srinivas2012,Chowdhury2017a}.  \looseness=-1

Furthermore, these bounds are based on the posterior variance function whose 
behavior for an increasing number of training data is not well understood. 
Especially when data points are added on-line, e.g. during the control 
tasks, compare~\citep{Umlauft2020}, the posterior variance has barely been 
analyzed 
formally due to a lack of suitable bounds.  Generally, there is only
limited understanding of the interaction between 
learning and control in feedback systems, which is 
crucial to provide guarantees for the control error and mainly motivates the 
work in this article. \looseness=-1

\subsection{Contribution}
The main contribution of this article is the derivation of a novel 
uniform error bound for GPs, which is guaranteed to converge to $0$ 
in the limit of infinitely many, suitably distributed training samples, 
and allows safety guarantees for the control of unknown dynamical systems. 
For the derivation of this uniform error bound, we proceed as follows:
First, we derive a bound for the posterior variance of GPs with
Lipschitz continuous kernels and propose improvements
for a more specific class of kernels. Additionally, we derive sufficient 
conditions for the distribution of the data to ensure the 
convergence of our bound  to zero. 
Second, we present a uniform error bound for GPs, which requires less prior 
knowledge and assumptions than in prior work as it is based on Lipschitz 
continuity. This uniform error bound is based on preliminary work presented
in \citep{Lederer2019}. We analyze bounds for the derivatives of GP sample functions
and employ them as constraints in likelihood maximization in order to 
ensure that the prior distribution fits to the available knowledge of the unknown 
function. Furthermore, we use the bounds on sample functions and the posterior 
variance to investigate the asymptotic behavior of the uniform error bound, which 
shows that arbitrarily small error bounds can be achieved under weak assumptions 
on the training data. The proposed GP bounds are employed to derive safety 
guarantees for controlling unknown dynamical systems based on Lyapunov 
theory~\citep{Khalil2002}. We illustrate all our results in simulations and 
demonstrate the superiority of our bounds for the GP posterior variance compared
to state-of-the-art methods. Finally, we demonstrate the safety of our control 
approach in an experiment with a robotic manipulator.\looseness=-1

\subsection{Structure of this article}
This article is structured as follows: We briefly introduce 
Gaussian process regression and provide a in-depth discussion of the related 
work in \cref{sec:background}. An analysis of the posterior variance bound is 
presented in \cref{sec:postVar}, followed by the derivation of a uniform error 
bound for GPs in \cref{sec:errorbound} using the probabilistic Lipschitz 
constant. In \cref{sec:safety} we demonstrate how the results of the former two 
section can be employed to prove safety of a model-based control law. A 
numerical 
illustration and comparison to prior work are provided in \cref{sec:numEval}, 
followed by a conclusion in \cref{sec:conclusion}.

\subsection{Notation}
Vectors/matrices
are denoted by lower/upper case bold symbols, the 
$n\times n$ identity matrix by~$\bm{I}_n$, the Euclidean norm 
by~$\|\cdot\|$, and $\lambda_{\min}(\bm{A})$ and $\lambda_{\max}(\bm{A})$ the
minimum and maximum eigenvalues of a matrix $\bm{A}$, respectively. Sets are 
denoted by upper case black board bold letters, and sets restricted to positive/non-negative 
numbers have an indexed~$+$/$+,0$, e.g.~$\mathbb{R}_+$ for 
all positive real valued numbers. The cardinality of 
sets is denoted by~$|\cdot|$ and subsets/strict subsets are indicated by \mbox{$\subset/\subseteq$}. 
The expectation operator
$E[\cdot]$ can have an additional index to specify 
the considered random variable. 
Class~$\mathcal{O}$ notation is used to provide 
asymptotic upper bounds on functions. Similarly, we
use class $\mathcal{K}_\infty$ functions 
$\alpha:\mathbb{R}_{+,0}\rightarrow\mathbb{R}_{+,0}$, 
which are strictly increasing, $\alpha(0)=0$ and satisfy
$\lim_{r\rightarrow\infty}\alpha(r)=\infty$.
The ceil and floor
operator are denoted by $\lceil\cdot\rceil$ and $\lfloor\cdot\rfloor$,
respectively. 
\looseness=-1

\section{Related Work and Background}
\label{sec:background}
\subsection{Gaussian Process Regression}
A Gaussian process is a stochastic process such 
that any finite number of outputs\linebreak
\mbox{$\{y_1,\ldots,y_M\} \subset\mathbb{R}$}
is assigned a joint Gaussian distribution with
prior mean function~$m:\mathbb{R}^d\rightarrow \mathbb{R}$ and covariance 
defined through the kernel 
$k:\mathbb{R}^d\times\mathbb{R}^d\rightarrow\mathbb{R}$
\citep{Rasmussen2006}. Therefore, the training outputs 
$y^{(i)}$ can be considered as observations of a sample 
function~$f:\mathbb{X}\subset\mathbb{R}^d\rightarrow\mathbb{R}$ of the GP 
distribution perturbed by i.i.d. zero mean Gaussian noise
with variance~$\sigma_n^2$. Without any specific information about the 
unknown function $f(\cdot)$ such as an approximate model, the prior mean 
function is typically set to $0$. We also assume this in the following
without loss of generality. Regression is performed by conditioning the 
prior GP distribution on the training data 
$\mathbb{D}_N=\{(\bm{x}^{(i)},y^{(i)})\}_{i=1}^N$ and a test point~$\bm{x}$. 
The conditional posterior 
distribution is again Gaussian and can be calculated analytically.
For this reason, we define the kernel matrix~$\bm{K}_N$ and the 
kernel vector~$\bm{k}_N(\bm{x})$ through~$K_{N,ij}=k(\bm{x}^{(i)},\bm{x}^{(j)})$ 
and~$k_{N,i}(\bm{x})=k(\bm{x},\bm{x}^{(i)})$, respectively, with~$i,j=1,\ldots,N$. 
Then, the posterior
mean~$\mu_N(\cdot)$ and variance~$\sigma_N^2(\cdot)$ are given by
\begin{align}
\mu_N(\bm{x})&=\bm{k}_N^T(\bm{x})\bm{A}_N^{-1}\bm{y}_N,\\
\sigma_N^2(\bm{x})&=k(\bm{x},\bm{x})-
\bm{k}_N^T(\bm{x})\bm{A}_N^{-1}\bm{k}_N(\bm{x}),
\label{eq:var}
\end{align}
where~$\bm{A}_{N}=\bm{K}_N+\sigma_n^2\bm{I}_{N}$
denotes the data covariance matrix and~$\bm{y}_N = [y^{(1)} \cdots y^{(N)}]^T$. 

The kernel typically depends on so called hyperparameters $\bm{\kappa}$ which allow 
to shape the prior distribution without changing properties of the 
sample functions such as periodicity or stationarity. Although there is a variety of
methods available for hyperparameter tuning \citep{Rasmussen2006}, in control-oriented 
applications they are often fitted to the training data by maximizing the marginal 
log-likelihood $P(\bm{y}_N|\bm{X},\bm{\kappa})$, which is commonly performed using
gradient-based optimization methods.

\subsection{Posterior Variance Bounds of Gaussian Processes}
Posterior variance bounds are well known for many methods closely 
related to Gaussian process regression. For noise-free interpolation,
the posterior variance has been analyzed using spectral 
methods~\citep{Stein1999}. While the asymptotic behavior
can be analyzed efficiently
with such methods, they are not suited to bound the 
posterior variance for specific training data sets. In the 
context of noise-free interpolation, many bounds from the 
area of scattered data approximation can be applied due to 
the equivalence of the posterior variance and the power
function~\citep{Kanagawa2018}. Therefore, classical results
\citep{Wu1993, Wendland2005, Schaback2006}
as well as newer findings~\citep{Beatson2010, Scheuerer2013}
can be directly used for GP interpolation. 
However, it is typically not clear how these results can
be generalized to regression with noisy observations.\looseness=-1

Posterior variance bounds for GP regression have mostly been 
developed as intermediate results for more
complex problems. For example, a variance bound has been developed
for GPs with isotropic kernels in the context of Bayesian 
optimization~\citep{Shekhar2018} while  
bounds for general kernels have been investigated within the 
analysis of average learning curves~\citep{Opper1999, Williams2000}
and experimental design~\citep{Wang2018}. Although these bounds
are well-suited for low data regimes, they fail to capture the 
asymptotic behavior. Therefore, upper bounds on the posterior variance
are missing which allow to describe the learning behavior over the
whole range of training data densities.

\subsection{Error Bounds of Gaussian Process Regression}
Uniform error bounds play a crucial role in quantifying the precision of a 
function approximator. For the case of noise free data, results of scattered 
data approximation with radial basis functions can be applied to
derive such bounds~\citep{Wendland2005} and translate to GP regression with
stationary kernels. Using Fourier transform methods, the classical result in 
\citep{Wu1993} derives error bounds for functions in the reproducing 
kernel Hilbert space (RKHS) associated with the interpolation kernel. By 
utilizing further properties of the RKHS, a uniform error bound with 
increased convergence rate is derived in \citep{Schaback2002}. These bounds are 
driven by the power function, which are - under certain conditions - equivalent 
to the GP posterior standard deviation~\citep{Kanagawa2018}. 

Extending scattered data interpolation to noisy observations leads to the 
concept of regularized kernel regression~\citep{Kanagawa2018}. For squared 
cost functions, this kernel ridge regression is identical to the GP posterior 
mean function~\citep{Rasmussen2006}.
The corresponding error bounds, e.g., in~\citep{Mendelson2002} depend on the 
empirical $\mathcal{L}_2$ covering number and the norm of the unknown function 
in the associated RKHS. With empirical $\mathcal{L}_2$ covering numbers, 
tighter 
error bounds can be derived under mild assumptions~\citep{Shi2013}. For general 
regularization, error bounds are 
derived in \citep{Dicker2017} as a function of the regularization and the RKHS 
norm of the function. 

Uniform error bounds depending on the maximum information gain and the RKHS 
norm for GP regression were derived in \citep{Srinivas2012}. However, these 
results only apply to bounded sub-Gaussian observation noise, which is a 
limitation compared to regularized kernel regression. To analyze the regret of 
an upper 
confidence bound algorithm in multi-armed bandit setting, an improved bound 
is derived in \citep{Chowdhury2017a}. Although these 
bounds are widely applied in safe reinforcement learning and control, they 
suffer from the following drawbacks: i) They depend on constants which are 
very 
difficult to calculate, which is no problem for the theoretical analysis, but 
prevents the application of the bounds in real-world tasks. ii) RKHS approaches 
face a general problem: The smoother the kernel, the smaller is the space of 
functions for which the bounds holds~\citep{Narcowich2006}. The RKHS attached 
to a covariance kernel is small compared to the support 
of the prior distribution of a GP~\citep{VanderVaart2011}.

Therefore, there is a lack of explicitly computable uniform error bounds, 
which do not rely on RKHS theory and the corresponding issues described previously. 
In order to avoid 
the difficulties that come with the RKHS view, we utilize the prior GP
distribution to derive error bounds for GP regression with noisy observations.

\section{Bounding the Posterior Variance of Gaussian Processes}
\label{sec:postVar}

Despite a wide variety of literature 
on average posterior variance 
bounds for isotropic kernels, data-dependent posterior variance bounds 
for general kernels have gained
far less attention. We derive in
\cref{subsec:varBound} an upper bound on the posterior 
variance, which depends on the number of samples in 
the neighborhood of the test point~$\bm{x}$. In \cref{subsec:probDist} we 
derive sufficient 
conditions on probability distributions of the training 
data that ensure the convergence of our bound. 

\subsection{Posterior Variance Bound}
\label{subsec:varBound}

The central idea in deriving an upper bound for 
the posterior variance of a GP lies 
in the observation that data close to a test point 
$\bm{x}$ usually lead to the highest decrease in the posterior 
variance. Therefore, it is natural to consider only 
training data close to the test point in the bound as 
more and more data is acquired. The following theorem 
formalizes this idea. The proofs for all the following
theoretical results can be found in the appendix.

\begin{theorem}
	\label{th:var_bound}
	Consider a GP with Lipschitz 
	continuous kernel~$k(\cdot,\cdot)$ with Lipschitz constant~$L_k$, an input 
	training data set~$\mathbb{D}_{N}^x=\{\bm{x}^{(i)}\}_{i=1}^{N}$ 
	and observation noise variance~$\sigma_n^2$. Let 
	$\mathbb{B}_{\rho}(\bm{x}^*)=\{ 
	\bm{x}'\in\mathbb{D}_N^x:~\|\bm{x}'-\bm{x}^*\|\leq 
	\rho \}$ denote the training data set restricted to a 
	ball around~$\bm{x}^*\in\mathbb{X}$ with radius~$\rho$. Then, for 
	each~$\bm{x}\in\mathbb{X}$, $\bm{x}^*\in\mathbb{X}$ and 
	$\rho\leq k(\bm{x},\bm{x})/L_k$, 
	the posterior variance is bounded by\looseness=-1
	\begin{align}
	\sigma_N^2(\bm{x})\leq \frac{(2L_k\rho(k(\bm{x},\bm{x})\!+\!k(
	\bm{x}^*\!,\bm{x}))\!-\!L_k^2\rho^2)\left|
	\mathbb{B}_{\rho}(\bm{x}^*)\right|
		\!+\!\sigma_n^2k(\bm{x},\bm{x})}{\left|
		\mathbb{B}_{\rho}(\bm{x}^*)\right|(k(\bm{x}^*\!,\bm{x}^*)
		\!+\!2L_k\rho)\!+\!\sigma_n^2}.
	\label{eq:sigbound}
	\end{align}
\end{theorem}
The point $\bm{x}^*$ is a reference point which can be
set such that the bound is minimized, e.g., by choosing 
the point of maximal variance in linear kernels or by choosing 
a point with many training samples in its proximity close to the 
test point $\bm{x}$ in stationary kernels. 
The parameter~$\rho$ can be interpreted as information
radius, which defines how far away from a reference point $\bm{x}^*$
training data is considered to be informative. However, this information radius is conservative
as all the data points with smaller radius are treated in
the theorem as if they had a distance of~$\rho$ to the 
test point. Therefore, a large~$\rho$ has the advantage
that many training points are considered, while a small
$\rho$ is beneficial if sufficiently many training samples
are close to the reference point~$\bm{x}^*$.

Note, that Theorem~\ref{th:var_bound} is very general as it is merely 
restricted 
to Lipschitz continuous kernels, which is a common 
property of kernels for regression~\citep{Rasmussen2006}. 
This generality comes at the price of tightness of the 
bound and tighter bounds exist under additional assumptions, 
e.g., the bound in~\citep{Shekhar2018} for isotropic, decreasing kernels,
which have non-positive derivatives 
$\frac{\partial}{\partial \tau}k(\tau)\leq 0$,~$\tau\geq 0$. However, this 
bound can directly be derived 
from \cref{th:var_bound}, which leads to the following corollary. 
\begin{corollary}
	\label{cor:var_bound}
	Consider a GP with isotropic, decreasing 
	covariance kernel~$k(\cdot)$, an input 
	training data set~$\mathbb{D}_{N}^x=
	\{\bm{x}^{(i)}\}_{i=1}^{N}$ and observation noise 
	variance~$\sigma_n^2$. Let 
	$\mathbb{B}_{\rho}(\bm{x})=\{ \bm{x}'
	\in\mathbb{D}_N^x:~\|\bm{x}'-\bm{x}\|\leq \rho \}$ 
	denote the training data set restricted to a ball 
	around~$\bm{x}$ with radius~$\rho$. Then, 
	for each~$\bm{x}\in\mathbb{X}$, the 
	posterior variance is bounded by
	\begin{align}
	\sigma_N^2(\bm{x})\leq k(0)-\frac{k^2(\rho)}{k(0)
		+\frac{\sigma_n^2}{|\mathbb{B}_{\rho}(\bm{x})|}}.
	\label{eq:isobound}
	\end{align}
\end{corollary}
Although Corollary~\ref{cor:var_bound} considers isotropic kernels, it
can be straightforwardly extended to kernels with automatic relevance 
determination~\citep{Neal1996}. This can be achieved by replacing the 
restriction to a ball by an ellipsoid, which leads to a set 
$\tilde{\mathbb{B}}_{\rho}(\bm{x}^*)=\{ \bm{x}'\in\mathbb{D}_N^x: 
(\bm{x}'-\bm{x}^*)^T\mathrm{diag}(\bm{l}^2)^{-1}(\bm{x}'-\bm{x}^*)\leq \rho^2 \}$.
Therefore, the posterior variance of many commonly used stationary kernels can 
be efficiently bounded based on Corollary~\ref{cor:var_bound}.

\subsection{Asymptotic Analysis}
\label{subsec:probDist}
In addition, \cref{th:var_bound} can also be used 
for determining an asymptotic decay rate of the posterior variance 
for~$\lim_{N\rightarrow\infty}\sigma_N^2(\bm{x})$.
Even though the limit of infinitely many training
data cannot be reached in practice, this analysis
is important because it helps to determine the 
amount of training data which is necessary to 
achieve a desired posterior variance. In the following
corollary, we provide necessary conditions that
ensure the convergence to zero of the bound~\eqref{eq:sigbound}.
\begin{theorem}
	\label{th:varvan}
	Consider a GP with Lipschitz 
	continuous kernel~$k(\cdot,\cdot)$ , an infinitely large 
	input training data set~$\mathbb{D}_{\infty}^x=\{ \bm{x}^{(i)} 
	\}_{i=1}^{\infty}$ and the observation noise 
	variance~$\sigma_n^2$. Let 
	$\mathbb{D}_{N}^x=\{ \bm{x}^{(i)} \}_{i=1}^{N}$ 
	denote the subset of any~$N$ input training 
	samples and let~$L_k$ be the Lipschitz constant 
	of kernel~$k(\cdot,\cdot)$. Furthermore, let 
	$\mathbb{B}_{\rho}(\bm{x})=\{ \bm{x}'
	\in\mathbb{D}_N^x:~\|\bm{x}'-\bm{x}\|\leq \rho \}$ 
	denote the training data set restricted to a ball 
	around~$\bm{x}$ with radius~$\rho$. If there 
	exists a function 
	\mbox{$\rho: \mathbb{N}\rightarrow \mathbb{R}_+$} and 
	a class $\mathcal{K}_{\infty}$ function $\alpha:\mathbb{R}\rightarrow\mathbb{R}_+$
	such that 
	\begin{align*}
	\rho(N)&\leq \frac{k(\bm{x},\bm{x})}{L_k}
	\quad \forall N\in\mathbb{N}\\
	\rho(N)&\in\mathcal{O}\left( \frac{1}{\alpha(N)} \right)\\
	\frac{1}{\left| \mathbb{B}_{\rho(N)}(\bm{x}) \right|}&\in\mathcal{O}\left( \frac{1}{\alpha(N)} \right)
	\end{align*}
	holds, the posterior variance at~$\bm{x}$ 
	converges to zero as follows
	\begin{align*}
	\sigma_N(\bm{x})\in\mathcal{O}\left( \frac{1}{\sqrt{\alpha(N)}} \right).
	\end{align*}
\end{theorem}

Although it might seem impractical that the number of 
training samples in a ball with vanishing radius has 
to reach infinity in the limit of infinite training 
data, this is not a restrictive condition. 
Deterministic sampling strategies can satisfy it, 
e.g. if a constant fraction of the samples lies on 
the considered point~$\bm{x}$ or if the maximally allowed 
distance of new samples reduces with the total number 
of samples. 

\begin{remark}
	\cref{th:varvan} does not require dense sampling in a neighborhood of the test 
	point~$\bm{x}$. In fact, the 
	conditions on the training samples in \cref{th:varvan}
	are satisfied if the data is sampled densely, e.g., from 
	a lower-dimensional manifold which contains the test point~$\bm{x}$, 
	such as a line through~$\bm{x}$.
\end{remark}

In the following, we derive conditions on sampling 
distributions that ensure a vanishing posterior 
variance bound. For fixed~$\rho$ it is well known that the number 
of training samples inside the ball~$\mathbb{B}_{\rho}
(\bm{x})$ converges to its expectation due to the 
strong law of large numbers. Therefore, it is 
sufficient to analyze the asymptotic behavior of the 
expected number of samples inside the ball instead of 
the actual number for fixed~$\rho$. However, it is not 
clear how fast the radius~$\rho(N)$ is allowed to 
decrease in order to ensure convergence of 
$|\mathbb{B}_{\rho(N)}(\bm{x})|$ to its expected 
value. The following lemma shows that the 
admissible order of~$\rho(N)$ depends on the local 
behavior of the density~$p(\cdot)$ around~$\bm{x}$.

\begin{lemma}
	\label{th:ball}
	Consider a sequence of points~$\mathbb{D}_{\infty}^x=
	\{ \bm{x}^{(i)} \}_{i=1}^{\infty}$ which is generated by 
	drawing from a probability distribution with 
	density~$p(\cdot)$. If there exists a 
	non-increasing function~$\rho: \mathbb{N}
	\rightarrow \mathbb{R}_+$ and constants~$c,\epsilon\in\mathbb{R}_+$ 
	such that
	\begin{align}
	\lim\limits_{N\rightarrow\infty}\rho(N)&=0\\
	\int\limits_{\{\bm{x}'\in\mathbb{X}:\|\bm{x}
		-\bm{x}'\|\leq \rho(N)\}}p(\bm{x}')\mathrm{d}\bm{x}'
	&\geq cN^{-1+\epsilon},
	\label{eq:cond2newmain}
	\end{align}
	then, the sequence 
	$|\mathbb{B}_{\rho(N)}(\bm{x})|$ asymptotically behaves as~$\mathcal{O}\left( N^{\epsilon} \right)$ $a.s.$
\end{lemma}

Similarly to \cref{th:var_bound}, Lemma~\ref{th:ball}
is formulated very general to be applicable to a wide 
variety of probability distributions. However, under 
additional assumptions condition \eqref{eq:cond2newmain} 
can be simplified. This is exemplary shown for 
probability densities which are positive in a 
neighborhood of the considered point~$\bm{x}$.\looseness=-1

\begin{corollary}
	\label{cor:varvan}
	Consider a sequence of points~$\mathbb{D}_{\infty}^x
	=\{ \bm{x}^{(i)} \}_{i=1}^{\infty}$ which is generated by 
	drawing from a probability distribution with 
	density~$p(\cdot)$, such that~$p(\cdot)$ is 
	positive in a ball around~$\bm{x}$ with any radius 
	$\xi\in\mathbb{R}_+$, i.e.
	\begin{align*}
	p(\bm{x}')>0\quad\forall \bm{x}'\in\{ \bm{x}': 
	\|\bm{x}-\bm{x}'\|\leq \xi  \}.
	\end{align*}
	Then, for all non-increasing 
	functions~$\rho:\mathbb{N}\rightarrow\mathbb{R}_+$ 
	for which exist~$c,\epsilon\in\mathbb{R}_+$ such 
	that
	\begin{align*}
	\rho(N)&\geq cN^{-\frac{1}{d}+\epsilon}\quad 
	\forall N\in\mathbb{N}\\
	\lim\limits_{N\rightarrow\infty}\rho(N)&=0
	\end{align*}
	it holds that~$|\mathbb{B}_{\rho(N)}
	(\bm{x})|\in\mathcal{O}\left( N^{\frac{\epsilon}{d}} \right) ~ a.s.$
\end{corollary}

This corollary shows that it is relatively 
simple to allow the maximum decay rate of 
$\rho(N)\approx N^{-1}$ for scalar inputs. 
For higher dimensions~$d$ however, it cannot 
be achieved and the allowed decay rate decreases 
exponentially with~$d$. Yet, this is merely a 
consequence of the curse of dimensionality.\looseness=-1

\section{Probabilistic Uniform Error Bound}
\label{sec:errorbound}

In the case of noise free observations or under the restriction to 
subspaces of a RKHS, probabilistic uniform error bounds are widely used in 
Gaussian process regression. However, RKHS based assumptions can be 
difficult to interpret and the involved 
constants usually have to be approximated using heuristics. Furthermore,
a subspace of a RKHS is an unnecessarily small hypothesis space since
the inherent probability distribution of GPs has larger support. Due to 
these reasons, we derive a easily computable and interpretable uniform 
error bound for Gaussian process regression with noisy observations. 
In Section~\ref{subsec:RKHS} we discuss the assumption of GP sample 
functions and relate it to RKHS interpretations. Exploiting well 
known properties of the distribution of maxima of Gaussian 
processes, we develop a method to determine interpretable hyperparameter
bounds based on prior system knowledge in Section~\ref{subsec:prob Lipschitz}. 
A uniform error bound is derived under the weak assumption of Lipschitz 
continuity of the unknown function and the covariance kernel in 
Section~\ref{subsec:regerror}. Finally, the asymptotic behavior of the 
uniform error bound is analyzed in Section~\ref{subsec:asymptotics}

\subsection{Gaussian Process Sample Functions and Spaces with Bounded RKHS}
\label{subsec:RKHS}

In order to understand main differences between our approach and existing
error bounds based on RKHS theory, we compare the sample space of a GP 
and the RKHS attached to a kernel $k(\cdot,\cdot)$. A key 
role in this analysis plays Mercer's theorem \citep{Mercer1909} which
guarantees the existence of a feature map $\phi_i(\bm{x})$, $i=1,\ldots,\infty$
such that we can express the kernel as
\begin{align*}
k(\bm{x},\bm{x}')=\sum\limits_{i=1}^{\infty}\lambda_i\phi_i(\bm{x})\phi(\bm{x}')
\end{align*}
with some $\lambda_i\in\mathbb{R}$ and $\bm{x},\bm{x}'\in\mathbb{X}\subset\mathbb{R}^d$ with a compact set $\mathbb{X}$. 
The linear span of this feature map
\begin{align*}
\mathcal{H}_0=\{ f(\bm{x}):~\exists N\in\mathbb{N},f_i\in\mathbb{R},i=1,\ldots,N \text{ such that } 
f(\bm{x})=\sum\limits_{i=1}^{N}f_i\phi_i(\bm{x}) \},
\end{align*}
which can be made a pre-Hilbert space by defining an inner product
$<f(\cdot),g(\cdot)>_{\mathcal{H}_0}=\sum_{i=1}^N\frac{f_ig_i}{\lambda_i}$. 
The RKHS $\mathcal{H}_k$ is then defined as the closure of $\mathcal{H}_0$ 
under the norm induced by the inner product $<\cdot,\cdot>_{\mathcal{H}_0}$, 
i.e., $\mathcal{H}_k=\overline{\mathcal{H}_0}$. The RKHS norm is defined as
$\|f(\cdot)\|_{\mathcal{H}_k}=\sum_{i=1}^{\infty}\frac{f_i^2}{\lambda_i}$
allowing the following typical assumption used to derive uniform error 
bounds.
\begin{assumption}
	The unknown function $f(\cdot)$ has a bounded norm in the RKHS $\mathcal{H}_k$
	attached to the kernel $k(\cdot,\cdot)$, i.e., $\|f(\cdot)\|_{\mathcal{H}_k}\leq B$
	for some $B\in\mathbb{R}_+$.
\end{assumption}
Since the RKHS norm captures the smoothness as well as the amplitude of a 
function~\citep{Kanagawa2018}, this assumption does not only restrict the function 
class that are in the hypothesis space, e.g., analytic functions for squared exponential 
kernels \citep{VanderVaart2011}, but also other properties such as, e.g., the Lipschitz constant. Therefore, 
this assumption exhibits several issues in practice such as its unnecessary restrictiveness
and the fact that determining a suitable bound $B$ can be difficult even
if the function $f(\cdot)$ is known. These issues are inherited by uniform error 
bounds based on this assumption.
In order to overcome these problems we consider the following assumption 
in the subsequent analysis.
\begin{assumption}
	\label{ass:samplefun}
	The unknown function $f(\cdot)$ is a sample from the 
	Gaussian process \linebreak$\mathcal{GP}(0,k(\bm{x},\bm{x}'))$.
\end{assumption}
This assumption can be reduced to two major components: the hypothesis space
and the probability distribution over this space.
Similar to the RKHS the space of sample functions can be defined based 
on the feature map $\phi_i(\bm{x})$, $i=1,\ldots,\infty$ as
\begin{align*}
\mathcal{S}=\left\{ f(\bm{x}):~\exists f_i,i=1,\ldots,\infty \text{ such that } 
f(\bm{x})=\sum\limits_{i=1}^{\infty}f_i\phi_i(\bm{x}) \right\}.
\end{align*}
The difference to the RKHS definition can be clearly seen: instead of the closure
of a pre-Hilbert space $\mathcal{H}_0$ the sample space $\mathcal{S}$ is the span
of the infinite dimensional feature map $\phi_i(\cdot)$, $i=1,\ldots,\infty$. 
This leads to the fact that the sample space is far "larger" than the RKHS. 
In fact, the RKHS $\mathcal{H}_k$ is included in the sample space $\mathcal{S}$
but has zero measure under the GP distribution \citep{Rasmussen2006}. For example,
the sample space $\mathcal{S}$ of the squared exponential kernel is equal to the space of 
continuous functions which contains the space of analytic functions \citep{VanderVaart2011}.
As shown in~\citep{Kanagawa2018} the sample space $\mathcal{S}$ can indeed be 
interpreted as a RKHS itself.

The probability distribution over the sample space $\mathbb{S}$ is defined through a zero mean 
Gaussian distribution over the coefficients $f_i$ with variance $\lambda_i$. Since
the $\lambda_i$s can be interpreted as the eigenvalues of a linear operator defined
through the covariance kernel $k(\cdot,\cdot)$ they unfortunately depend on the 
hyperparameters which are typically unknown a priori. However, this dependence 
is equally strong for the RKHS norm $\|f(\cdot)\|_{\mathcal{H}_k}$ which causes
the problem that modification of the hyperparameters can lead to violation of the 
bound $\|f(\cdot)\|_{\mathcal{H}_k}\leq B$ with fixed $B$. In contrast, variation in 
the hyperparameters only changes the probability assigned to functions but not 
the hypothesis space itself, e.g., strongly varying continuous functions are assigned a low 
probability in Gaussian processes with squared exponential kernel with large length 
scale but the probability is always positive. Therefore, the hyperparameters are 
important to shape our prior distributions such that it indeed represents our 
a priori knowledge of the unknown function but Assumption~\ref{ass:samplefun} is
valid independently of them. 

In addition to the hypothesis space and the corresponding probability distribution, 
observation noise plays a central role in Gaussian process regression. We consider 
natural noise of the GP framework in the following.
\begin{assumption}
	Observations $y=f(\bm{x})+\epsilon$ are perturbed 
	by zero mean i.i.d. Gaussian noise $\epsilon$ with variance 
	$\sigma_n^2$.
	\label{ass:noise}
\end{assumption}
This assumption is in contrast to the bounded, sub-Gaussian noise requirement posed 
in \citep{Srinivas2012, Chowdhury2017a}. Although our assumption can be considered 
restrictive because only Gaussian noise is allowed, the requirement to know an upper
bound of the noise is equally restrictive in practice. However, Gaussian noise is a 
well-established assumption and frequently employed in control theoretical literature,
e.g., \citep{Kailath1968, Hadidi1979, Ding2010, Komaee2012}.

\subsection{Belief Shaping through Constrained Likelihood Optimization}
\label{subsec:prob Lipschitz}

Hyperparameters of covariance kernels play a crucial role in Gaussian 
process regression since they determine the shape of the probability 
distribution over the function space. Thus, they strongly affect prediction
performance and are critical in encoding prior knowledge. Despite of this 
importance, hyperparameters are typically determined by an unconstrained optimization 
of the log-likelihood of the data. This can lead to poor models and 
overestimation of the model confidence through small posterior variances, 
particularly in regions of the input space with few training samples. In order
to overcome this issue, we propose a constrained hyperparameter optimization,
which allows to take into account prior knowledge of the unknown function in 
the form of uncertain estimates of extrema and derivative extrema. This knowledge
is often available, e.g., for physical systems. 

In order to see 
a direct relationship between the hyperparmeters and these function properties 
more clearly, the following result is introduced which relates the 
extremum value of sample functions to the 
covariance kernel $k(\cdot,\cdot)$.
\begin{theorem}
	\label{th:fmax}
	Consider a zero mean Gaussian process defined through the
	covariance kernel $k(\cdot,\cdot)$ 
	with continuous partial derivatives up to the second order 
	and Lipschitz constant $L_k$ on the set $\mathbb{X}$ 
	with maximal extension 
	\mbox{$\theta=\max_{\bm{x},\bm{x}'\in\mathbb{X}}\|\bm{x}-\bm{x}'\|$}. 
	Then, with probability of at least $1-\delta_f$ a sample 
	function $f(\cdot)$ satisfies
	\begin{align*}
	\max\limits_{\bm{x}\in\mathbb{X}}f(\bm{x})\leq f_{\max}(\delta_f)
	\end{align*}
	with probabilistic maximum absolute value
	\begin{align}
		f_{\max}(\delta_f)=
		\left(\sqrt{2\log\left( \frac{1}{\delta_f} \right)}
		+12\sqrt{2d\log\left( \frac{\sqrt{4\theta L_k}(1\!+\!\sqrt{2})\mathrm{e}}{\max\limits_{\bm{x}\in\mathbb{X}}\sqrt{k(\bm{x},\bm{x})}} \right) }\right)\max\limits_{\bm{x}\in\mathbb{X}}
		\sqrt{k(\bm{x},\bm{x})}.
		\label{eq:fbound}
	\end{align}
\end{theorem}
Analogously, we can derive a bound for the partial derivatives of the sample
functions.
\begin{corollary}
	\label{th:Lip_f}
	Consider a zero mean Gaussian process defined through the
	covariance kernel $k(\cdot,\cdot)$ 
	with continuous partial derivatives up to the fourth order
	and partial derivative kernels 
	\begin{align*}
	k^{\partial i}(\bm{x},\bm{x}')&=\frac{\partial^2}{\partial x_i\partial x_i'}
	k(\bm{x},\bm{x}')\quad \forall i=1,\ldots, d.
	\end{align*}
	Let $L_k^{\partial i}$ denote the Lipschitz constants of the partial 
	derivative kernels $k^{\partial i}(\cdot,\cdot)$ on the set $\mathbb{X}$ 
	with maximal extension 
	\mbox{$\theta=\max_{\bm{x},\bm{x}'\in\mathbb{X}}\|\bm{x}-\bm{x}'\|$}. 
	Then, for each $i=1,\ldots,d$ a sample function $f(\cdot)$ 
	of the Gaussian process satisfies with probability of at least $1-\delta_L$ that
	\begin{align*}
		\max\limits_{\bm{x}\in\mathbb{X}}\frac{\partial}{\partial x_i}f(\bm{x})\leq f_{\max}^{\partial i}(\delta_L)
	\end{align*}
	with probabilistic maximum absolute derivative
	\begin{align}
		f_{\max}^{\partial i}(\delta_L)=\left(\sqrt{2\log\left( \frac{1}{\delta_L} \right)}
		+12	\sqrt{2d\log\left( \frac{\sqrt{4\theta L_k^{\partial i}}(1\!+\!\sqrt{2})\mathrm{e}}{\max\limits_{\bm{x}\in\mathbb{X}}\sqrt{k^{\partial i}(\bm{x},\bm{x})}} \right) }\right)\max\limits_{\bm{x}\in\mathbb{X}}
		\sqrt{k^{\partial i}(\bm{x},\bm{x})}.
		\label{eq:Lfbound}
	\end{align}
\end{corollary}
Although \cref{th:fmax} and Corollary~\ref{th:Lip_f} do not directly provide intuitive insight
into the relation between hyperparameters and maximal values, this is easily achieved by 
simplifying them for a specific kernel. This is straightforward for many commonly used covariance 
functions such as the squared exponential and Mat\'ern 
class kernels. For example, we consider the isotropic squared exponential kernel 
\begin{align*}
	k(\bm{x},\bm{x}')=\sigma_f^2\exp\left( -\frac{1}{2l^2}\|\bm{x}-\bm{x}'\|^2\right),
\end{align*}
with lengthscale $l\in\mathbb{R}_+$ and signal standard deviation $\sigma_f\in\mathbb{R}_+$. 
It is trivial to show that this covariance function satisfies
\begin{align*}
\max\limits_{\bm{x}\in\mathbb{X}}\sqrt{k(\bm{x},\bm{x})}&=\sigma_f&
\max\limits_{\bm{x}\in\mathbb{X}}\sqrt{k^{\partial i}(\bm{x},\bm{x})}&=\frac{\sigma_f}{l}.
\end{align*}
Furthermore, the Lipschitz constant of the kernel and the derivative kernels are 
given by
\begin{align*}
L_k&= \frac{\sigma_f^2}{l\mathrm{e}^{\frac{1}{2}}}&
L_k^{\partial i}&=\omega\frac{\sigma_f^2}{l^3},
\end{align*}
where
\begin{align*}
\omega=\sqrt{6(3-\sqrt{6})}\exp\left( \sqrt{\frac{3}{2}}-\frac{3}{2} \right).
\end{align*}
Based on these values, application of Theorem~\ref{th:fmax} and Corollary~\ref{th:Lip_f} 
to the squared exponential kernel leads to the bounds
\begin{align*}
f_{\max}(\delta_f)&=\left(\sqrt{2\log\left(\frac{1}{\delta_f}\right)}+12 \sqrt{2d\log\left( \frac{\sqrt{4\theta}(1+\sqrt{2})\mathrm{e}^{\frac{3}{4}}}{l} \right) }\right)\sigma_f\\
f_{\max}^{\partial i}(\delta_L)&=\left(\sqrt{2\log\left(\frac{1}{\delta_L}\right)}+12\sqrt{2d\log\left( \frac{\sqrt{4\theta \omega}(1+\sqrt{2})\mathrm{e}}{\sqrt{l}} \right) }\right)\frac{\sigma_f}{l}.
\end{align*}
These expressions are easily interpretable and provide deep insight into 
the effects of hyperparameters on the prior distribution. The maximum value
of sample functions is strongly influenced by the signal standard deviation 
$\sigma_f$, while the length scale $l$ merely acts as a logarithmic factor. 
This resembles our intuitive understanding of hyperparameters: choosing a 
large value for $\sigma_f$ causes functions that may differ strongly from 
the mean function. In contrast, the quotient between signal standard deviation 
$\sigma_f$ and length scale $l$ is crucial for the maximum absolute value of 
the partial derivative functions. This coincides with the observation that
a small length scale causes strongly varying functions. 

Due to this strong influence of the kernel hyperparameters on the distribution over 
sample functions, they have a large impact on the regression performance in regions 
with sparse data, and have a crucial importance regarding the proper estimation of 
the uncertainty of predictions. In order to ensure to obtain suitable hyperparameters,
we propose to include approximate information about the roughness of unknown functions in
the training, which is often available in practical applications. We assume that this 
knowledge can be expressed as
\begin{align*}
P\left(\max\limits_{\bm{x}\in\mathbb{X}}f(\bm{x})\geq \bar{f}\right)&\geq\delta_f & P\left(\min\limits_{\bm{x}\in\mathbb{X}}f(\bm{x})\leq \ubar{f}\right)&\geq\delta_f\\
P\left(\max\limits_{\bm{x}\in\mathbb{X}}\frac{\partial}{\partial x_i}f(\bm{x})\geq \bar{f}^{\partial i}\right)&\geq \delta_L & 
P\left(\min\limits_{\bm{x}\in\mathbb{X}}\frac{\partial}{\partial x_i}f(\bm{x})\leq \ubar{f}^{\partial i}\right)&\geq \delta_L,
\end{align*}
where $\bar{f}$, $\ubar{f}$, $\bar{f}^{\partial i}$ and $\ubar{f}^{\partial i}$ 
are the bounds representing 
the probabilistically known prior information of the unknown function $f(\cdot)$. 
Due to the bounds \eqref{eq:fbound}, \eqref{eq:Lfbound} and symmetry, a necessary condition
to satisfy this knowledge with the sample functions of the GP prior distribution is 
given by
\begin{align*}
	\max\{-\ubar{f},\bar{f}\}&\leq f_{\max}(\delta_f)\\
	\max\{-\ubar{f}^{\partial_i},\bar{f}^{\partial_i}\}&\leq f^{\partial_i}(\delta_L).
\end{align*}
Therefore, it is straightforward to add constraints to a training algorithm 
such as likelihood maximization as follows
\begin{subequations}
	\begin{align}
	\max\limits_{\kappa} &-\frac{1}{2}\bm{y}^T\bm{A}_{N,\kappa}^{-1}\bm{y}- \frac{1}{2}\log(|\bm{A}_{N,\kappa}|)-\frac{N}{2}\log(2\pi)\\
	\label{eq:fcon}
	\mathrm{subject ~to }~ & \max\{-\ubar{f},\bar{f}\}\leq f_{\max}(\delta_f)\\
	&\max\{-\ubar{f}^{\partial_i},\bar{f}^{\partial_i}\}\leq f^{\partial_i}(\delta_L)\qquad \forall i=1,\ldots,d.
	\label{eq:Lcon}
	\end{align}
\end{subequations}
Even though the hyperparameters obtained with the constrained 
log-likelihood maximization might result in a higher empirical prediction
error of the GP mean function, they define a probability distribution 
which reflects the assumptions on the unknown function. This implies that
the uncertainty is captured properly by the posterior GP variance, 
which can be challenging especially in regions with no training data. 
Thereby, the constrained log-likelihood maximization is a crucial step towards
a well-defined and interpretable uniform error bound.

\subsection{Uniform Error Bound based on Lipschitz Continuity}
\label{subsec:regerror}
Typical uniform error bounds \citep{Srinivas2012, Chowdhury2017a} rely on
subspaces of the RKHS attached to the covariance kernel $k(\cdot,\cdot)$
used in GP regression. In contrast, we derive a theorem assuming the 
unknown function lies in the sample space $\mathcal{S}$, which is a 
weaker assumption as discussed in Section~\ref{subsec:RKHS} and allows
to incorporate prior knowledge as described in Section~\ref{subsec:prob Lipschitz}. 
In order
to pursue our analysis we require the unknown function to 
exhibit a bounded Lipschitz constant, which is a weak assumption 
for many control systems and already required for determining interpretable 
hyperparameters of the covariance kernel. Since a Lipschitz continuous function
requires Lipschitz continuous kernels for meaningful regression, 
it is also trivial 
to show Lipschitz continuity of the posterior mean and variance. By defining 
a virtual grid for analysis we show that the corresponding Lipschitz constants
can be exploited to derive a uniform error bound in the following theorem.
\begin{theorem}
	\label{th:errbound_with}
	Consider a zero mean Gaussian process defined through the continuous
	covariance kernel $k(\cdot,\cdot)$ with Lipschitz constant $L_k$ 
	on the compact set $\mathbb{X}$. Furthermore, consider a continuous unknown 
	function $f:\mathbb{X}\rightarrow \mathbb{R}$ with Lipschitz constant $L_f$ 
	and $N\in\mathbb{N}$ observations $y^{(i)}$ satisfying Assumptions~\ref{ass:samplefun}
	and \ref{ass:noise}. Then, the 
	posterior mean $\nu_{N}(\cdot)$ and posterior variance $\sigma_N(\cdot)$ 
	of a Gaussian process conditioned on the training data $\{(\bm{x}^{(i)},y^{(i)})\}_{i=1}^N$
	are continuous with Lipschitz constants $L_{\nu_N}$ and $L_{\sigma_N^2}$
	on $\mathbb{X}$, respectively, where
	\begin{align*}
	L_{\nu_N}&\leq L_k\sqrt{N} 
	\left\| (\bm{K}(\bm{X}_N,\bm{X}_N)+\sigma_n^2\bm{I}_N)^{-1}\bm{y}_N \right\|\\
	L_{\sigma_N^2}&\leq 2\tau L_k\left( 1+N
	\|(\bm{K}(\bm{X}_N,\bm{X}_N)+\sigma_n^2\bm{I}_N)^{-1}\|
	\max\limits_{\bm{x},\bm{x}'\in\mathbb{X}}k(\bm{x},\bm{x}') \right).
	\end{align*}
	Moreover, pick $\delta\in (0,1)$, $\tau\in\mathbb{R}_+$ and set 
	\begin{align*}
	\beta(\tau)&=2\log\left(\frac{M(\tau,\mathbb{X})}{\delta}\right)\\
	\gamma(\tau)&=\left( L_{\nu_N}+L_f\right)\tau+\sqrt{\beta(\tau)L_{\sigma_N^2}\tau},
	\end{align*}
	where $M(\tau,\mathbb{X})$ denotes the $\tau$-covering number of $\mathbb{X}$, i.e., 
	the minimum number such there exists a set $\mathbb{X}_{\tau}$ 
	satisfying $|\mathbb{X}_{\tau}|=M(\tau,\mathbb{X})$ and $\forall\bm{x}\in\mathbb{X}$ there exists 
	$\bm{x}'\in\mathbb{X}_{\tau}$ with $\|\bm{x}-\bm{x}'\|\leq \tau$.
	Then, it holds that
	\begin{align}
	\label{eq:errorbound}
	P\left(|f(\bm{x})-\nu_{N}(\bm{x})|\leq 
	\sqrt{\beta(\tau)}\sigma_{N}(\bm{x})+\gamma(\tau), 
	~\forall\bm{x}\in\mathbb{X}\right)\geq 1-\delta.
	\end{align}
\end{theorem}

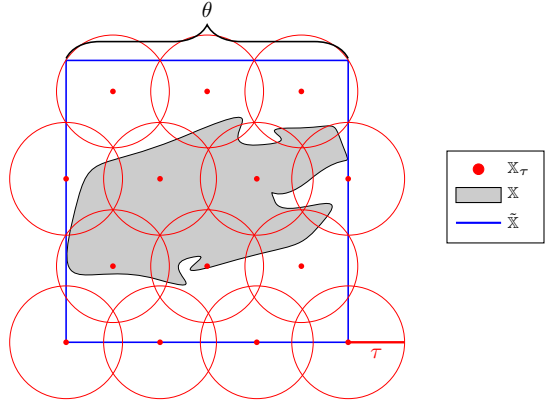
\begin{figure}
	\center
	\tikzsetnextfilename{cover_illustration}
	\begin{minipage}{0.4\textwidth}
		\center
		\scalebox{0.75}{
		\begin{tikzpicture}		
		\draw[smooth cycle, tension=2, fill=black!20] plot coordinates{(0,0) (1.5,-0.98) (2.1,-0.6) (2.7,-0.7) (4.5,0.1) (3.6,0.4) (4.24,0.8) (4.8,1.4) (3.9,1.7) (3.3,1.4) (2.1,1.6)};
		
		\draw[blue, thick] (-0.1,-2.1)--(4.9,-2.1)--(4.9,2.9)--(-0.1,2.9)--(-0.1,-2.1);
		
		\fill[red] (0.73,2.35) circle (0.05cm);
		\fill[red] (2.4,2.35) circle (0.05cm);
		\fill[red] (4.07,2.35) circle (0.05cm);
		
		\fill[red] (-0.1,0.8) circle (0.05cm);
		\fill[red] (1.565,0.8) circle (0.05cm);
		\fill[red] (3.28,0.8) circle (0.05cm);
		\fill[red] (4.9,0.8) circle (0.05cm);
		
		\fill[red] (0.73,-0.75) circle (0.05cm);
		\fill[red] (2.4,-0.75) circle (0.05cm);
		\fill[red] (4.07,-0.75) circle (0.05cm);
		
		\fill[red] (-0.1,-2.1) circle (0.05cm);
		\fill[red] (1.565,-2.1) circle (0.05cm);
		\fill[red] (3.28,-2.1) circle (0.05cm);
		\fill[red] (4.9,-2.1) circle (0.05cm);
		
		\draw[red] (0.73,2.35) circle (1);
		\draw[red] (2.4,2.35) circle (1);
		\draw[red] (4.07,2.35) circle (1);
		
		\draw[red] (-0.1,0.8) circle (1);
		\draw[red] (1.565,0.8) circle (1);
		\draw[red] (3.28,0.8) circle (1);
		\draw[red] (4.9,0.8) circle (1);
		
		\draw[red] (0.73,-0.75) circle (1);
		\draw[red] (2.4,-0.75) circle (1);
		\draw[red] (4.07,-0.75) circle (1);
		
		\draw[red] (-0.1,-2.1) circle (1);
		\draw[red] (1.565,-2.1) circle (1);
		\draw[red] (3.28,-2.1) circle (1);
		\draw[red] (4.9,-2.1) circle (1);
		
		\draw[red,very thick] (4.9,-2.1)--(5.9,-2.1);
		
		\node[red] at (5.4,-2.3) {$\tau$};
		\draw [thick,decorate, decoration={brace,amplitude=17pt,mirror,raise=1pt}, yshift=0pt]
		(4.9,2.9) -- (-0.1,2.9) node [black,midway,xshift=0.0cm,yshift=0.9cm] {
			$\theta$};
		\end{tikzpicture}
	}
	\end{minipage}
	\begin{minipage}{0.2\textwidth}
		\begin{tikzpicture}
			\begin{axis}[%
			hide axis,
			xmin=-5,
			xmax=5,
			ymin=0,
			ymax=1
			]
				\addlegendimage{red,only marks};
				\addlegendentry{$\mathbb{X}_{\tau}$};
				\addlegendimage{fill=black!20, area legend};
				\addlegendentry{$\mathbb{X}$};
				\addlegendimage{blue, thick};
				\addlegendentry{$\tilde{\mathbb{X}}$};
			\end{axis}
	   \end{tikzpicture}
	\end{minipage}	
	\caption{Illustration of the determination of an upper bound for the covering number~$M(\tau,\mathbb{X})$.}
	\label{fig:cover}
\end{figure}

It is important to note that most of the parameters in \cref{th:errbound_with}
do not require a difficult analysis such that the bound \eqref{eq:errorbound} 
can be directly evaluated. While the computation of the exact covering number 
$M(\tau,\mathbb{X})$ is a difficult problem for general sets $\mathbb{X}$,
it can be easily upper bounded as illustrated in \cref{fig:cover}. For this reason, 
we overapproximate the set 
$\mathbb{X}$ through a hypercube $\tilde{\mathbb{X}}$ with edge length $\theta$. 
Then, the covering number of $\tilde{\mathbb{X}}$ is bounded by \citep{Shalev-Shwartz2013} 
\begin{align*}
M(\tau,\tilde{\mathbb{X}})\leq \left(\frac{\theta\sqrt{d}}{2\tau}\right)^d,
\end{align*}
which is by construction also a bound for the covering number of $\mathbb{X}$, 
i.e., 
\begin{align*}
M(\tau,\mathbb{X})\leq \left(\frac{\theta\sqrt{d}}{2\tau}\right)^d.
\end{align*}

Since a probabilistic Lipschitz constant $L_f$ of the unknown function 
is a preliminary requirement for suitable hyperparameters 
as explained in Section~\ref{subsec:prob Lipschitz},
the uniform error bound mainly depends on training data $\mathbb{D}$, 
the considered set $\mathbb{X}$, the probability $\delta$ and the virtual grid 
constant $\tau$ allowing the direct computation of the bound \eqref{eq:errorbound}.
Hence, it is trivially evaluated with very little computational complexity emphasizing 
its high applicability, especially in safe control of unknown systems which 
can require fast computation of bounds.

Furthermore, it is important to observe that the virtual grid constant $\tau$ 
balances the effect of the state space discretization and the inherent uncertainty 
measured by the posterior standard deviation $\sigma_N(\cdot)$. Therefore,
$\gamma(\tau)$ can be made arbitrarily small by choosing s sufficiently fine
virtual grid. This in turn increases $\beta(\tau)$ and thus the 
effect of the posterior standard deviation $\sigma_N(\cdot)$ on the bound. 
However, $\beta(\tau)$ depends merely logarithmically on $\tau$ such that
even poor Lipschitz constants $L_{\nu_N}$, $L_{\sigma_N^2}$ and $L_f$ can
be easily compensated by small virtual grid constants $\tau$. In the following
subsection we exploit this trade-off to analyze the asymptotic behavior $N\rightarrow\infty$
of our uniform error bound and show that convergence to zero can be achieved with 
suitable training data.

\begin{remark}
	Since the standard deviation $\sigma_N(\cdot)$ varies within the state space $\mathbb{X}$,
	an optimal virtual grid constant $\tau$, which minimizes the expression 
	$\sqrt{\beta(\tau)}\sigma_{N}(\bm{x})+\gamma(\tau)$ for all $\bm{x}\in\mathbb{X}$, does 
	not exist in general. While simple approaches such as choosing $\tau$ such that 
	$\gamma(\tau)$ is negligible for all $\bm{x}\in\mathbb{X}$ provide satisfying 
	results in our simulations, more complex approaches remain open research questions.
\end{remark}

\subsection{Analysis of Asymptotic Behavior}
\label{subsec:asymptotics}

A crucial question in safe reinforcement learning and control of unknown systems is the existence
of lower bounds for the learning error since they determine the achievable control performance. 
It can be directly observed from Theorem~\ref{th:errbound_with} that the training 
data plays the critical role in this question since it affects the posterior standard deviation 
$\sigma_N(\cdot)$. By analyzing the asymptotic behavior of \eqref{eq:errorbound}, i.e., considering
the limit $N\rightarrow\infty$, we can derive conditions for the posterior variance which 
ensure that an arbitrarily small error can be achieved with suitable training data. 
This is shown in the following theorem.

\begin{theorem}
	\label{th:as_err}
	Consider a zero mean Gaussian process defined through the
	covariance kernel $k(\cdot,\cdot)$ with continuous partial derivatives up to the 
	fourth order on the set $\mathbb{X}$. Furthermore, consider an infinite data stream
	of observations $(\bm{x}^{(i)},y^{(i)})$ of an unknown, Lipschitz continuous function 
	\mbox{$f:\mathbb{X}\rightarrow \mathbb{R}$} which satisfies Assumptions~\ref{ass:samplefun}
	and \ref{ass:noise}.
	Let $\nu_N(\cdot)$ and $\sigma_N(\cdot)$ denote the mean and standard deviation 
	of the Gaussian process conditioned on the first $N$ observations. If there exists
	a class $\mathcal{K}_{\infty}$ function $\alpha:\mathbb{R}\rightarrow\mathbb{R}_+$ 
	and a constant~$\epsilon>0$ such that 
	\begin{align*}
		\sigma_N(\bm{x})&\in \mathcal{O}\left(\frac{1}{\alpha(N)}\right)\subset\mathcal{O}\left(\frac{1}{\sqrt{\log(N)}}\right)&\forall\bm{x}\in\mathbb{X},
	\end{align*}
	then the learning error uniformly converges to zero almost surely with rate
	\begin{align*}
		\sup\limits_{\bm{x}\in\mathbb{X}}\|\nu_N(\bm{x})-f(\bm{x})\|\in\mathcal{O}\left(\frac{\sqrt{\log(N)}}{\alpha(N)}\right)\quad \text{a.s.}
	\end{align*}
\end{theorem}
Compared to Theorem~\ref{th:errbound_with} this theorem does not require
knowledge of the Lipschitz constant $L_f$ of the unknown function $f(\cdot)$. 
However, it instead requires increased smoothness of the covariance kernel.
This smoothness guarantees the existence of probabilistic maximum absolute values 
and Lipschitz constants due to Theorem~\ref{th:fmax} and Corollary~\ref{th:Lip_f}, 
respectively, which is exploited together with Theorem~\ref{th:errbound_with}
in the proof. Moreover, note that this theorem guarantees almost sure convergence,
while previous versions in \citep{Lederer2019} merely guaranteed convergence with 
arbitrary probability. 

Although the asymptotic analysis is typically not performed
for the related approaches in \citep{Srinivas2012,Chowdhury2017a}, similar conditions
for the posterior variance can be straightforwardly derived. Due to the dependence 
on the information gain of these approaches, the necessary decrease rate depends
strongly on the covariance kernel $k(\cdot,\cdot)$. For example, the information
gain behaves as $\mathcal{O}\left( \log(N)^{d+1} \right)$ for the squared exponential kernel 
which leads to the conditions 
$\sigma_N(\cdot)\in\mathcal{O}\left( \log(N)^{-\frac{d}{2}-2} \right)$ for 
\citep{Srinivas2012} and $\sigma_N(\cdot)\in\mathcal{O}\left( 
\log(N)^{-\frac{d+1}{2}} \right)$ for \citep{Chowdhury2017a}. For the 
Mat\'ern kernel with smoothness parameter $\nu$ the information gain exhibits 
the asymptotic behavior $\mathcal{O}\left( N^{\frac{d(d+1)}{2\nu+d(d+1)}}\log(N) \right)$ 
such that $\sigma_N(\cdot)\in\mathcal{O}\left( N^{-\frac{d(d+1)}{4\nu+2d(d+1)}}\log(N)^{-2} \right)$ 
for \citep{Srinivas2012} and $\sigma_N(\cdot)\in\mathcal{O}\left( 
N^{-\frac{d(d+1)}{4\nu+2d(d+1)}}\log(N)^{-\frac{1}{2}} \right)$ for \citep{Chowdhury2017a} are required.
It can be clearly observed that in both cases the conditions on the 
posterior variance are far more restrictive than required by Theorem~\ref{th:as_err}.

In addition to the weaker conditions of Theorem~\ref{th:as_err}, it allows to directly 
pose a condition on the infinite training data sequence to ensure a vanishing uniform error 
bound by exploiting Theorem~\ref{th:varvan}. This is shown in the following corollary. 

\begin{corollary}
	\label{th:as_err_data}
	Consider a zero mean Gaussian process defined through the
	covariance kernel $k(\cdot,\cdot)$ with continuous partial derivatives up to the 
	fourth order on the set $\mathbb{X}$. Furthermore, consider an infinite sequence of 
	observations $y^{(i)}$ of an unknown, Lipschitz continuous function 
	\mbox{$f:\mathbb{X}\rightarrow \mathbb{R}$} which satisfies Assumptions~\ref{ass:samplefun}
	and \ref{ass:noise}.
	If there exists a function \linebreak$\rho:\mathbb{N}\rightarrow\mathbb{R}_+$ and a constant $\epsilon\in\mathbb{R}_+$ 
	such that the assumptions of \cref{th:varvan} are satisfied with \linebreak$1/\alpha(N)\in\mathcal{O}(\log(N)^{-1-\epsilon})$ for all $\bm{x}\in\mathbb{X}$,
	then the learning error uniformly converges to zero almost surely with rate
	\begin{align*}
	\sup\limits_{\bm{x}\in\mathbb{X}}\|\nu_N(\bm{x})-f(\bm{x})\|\in\mathcal{O}\left(\sqrt{\frac{\log(N)}{\alpha(N)}}\right)\quad \text{a.s.}
	\end{align*}
\end{corollary}

\section{Safety Guarantees for Control of Unknown Dynamical Systems}
\label{sec:safety}
When applying learning controllers in safety-critical applications like
autonomous driving or robots working in close proximity to humans, 
upper bounds for the tracking error are crucial to provide formal 
safety guarantees for the dynamical systems. Therefore, we demonstrate
how the results in the previous sections can be applied to design 
regulators for the safe control of unknown dynamical systems. A feedback
linearization controller for robust tracking is presented in \cref{subsec:track}.
In section~\ref{subsec:stab} the stability of this controller is analyzed and 
the asymptotic behavior for infinite training data is investigated.

\subsection{Tracking Control Design}
\label{subsec:track}
We consider a nonlinear control affine dynamical system
\begin{align}
\label{eq:sys}
\dot{x}_1  = x_2,\qquad \dot{x}_2  = x_3,\qquad \ldots \qquad \dot{x}_d = f(\bm{x}) + u,
\end{align}
with state~$\bm{x}=[x_1\ x_2\ \ldots\ x_d]^T \in \mathbb{X} \subset \mathbb{R}^d$ and control 
input~$u \in \mathbb{U} \subseteq \mathbb{R}$. Although we assume that the function~$f(\cdot)$
is unknown, the structure of the dynamics is not. However, this is not a severe 
restriction since nonlinear control affine systems of the 
form \eqref{eq:sys} can be applied in a wide variety of applications such as Lagrangian 
dynamics and many physical systems.

Our goal is to determine a policy~$\pi:\mathbb{X} \to \mathbb{U}$ such that the 
output~$x_1$ tracks the desired trajectory~$x_{\mathrm{ref}}(t)$ with vanishing tracking 
error~$\bm{e} =[e_1\ \ldots\ e_d]^T=\bm{x}-\bm{x}_{\mathrm{ref}}$ where
$\bm{x}_{\mathrm{ref}} =[x_{\mathrm{ref}}\ \dot{x}_{\mathrm{ref}}\ \ldots \overset{\scriptscriptstyle d}{\dot{x}}_{\mathrm{ref}}]^T$, 
i.e.,~$\lim_{t\to\infty} \|\bm{e}\|= 0$. 
We assume that training data of the real system is 
available in the form of noisy observations~$y^{(i)} = f(\bm{x}^{(i)}) + \epsilon$, $i=1,\ldots,N$, 
$\epsilon\sim\mathcal{N}(0,\sigma_n^2)$ such that we can train a Gaussian process and
use its posterior mean function $\nu_N(\cdot)$ as a model estimate of the unknown 
function $f(\cdot)$. Based on this model estimate we compensate the 
unknown non-linearity $f(\cdot)$ which is commonly referred to as feedback
linearization for control affine systems \citep{Khalil2002}. After compensation
of the non-linearity we apply linear regulators for the tracking such that
we obtain the policy
\begin{align}
\label{eq:FeliCtrl}
u =  -\nu_N(\bm{x}) +\pi(\bm{x}),
\end{align}
with the linear controller for tracking
\begin{align*}
\pi(\bm{x})=\overset{\scriptscriptstyle d+1}{\dot{x}}_{\mathrm{ref}}
-k_c  \begin{bmatrix}
\bm{\lambda}&1
\end{bmatrix}\bm{e}
\end{align*}
with control gain $k_c\in\mathbb{R}_+$ and filter coefficients $\bm{\lambda}=\begin{bmatrix}
\lambda_1&\lambda_2&\ldots&\lambda_{d-1}
\end{bmatrix}\in\mathbb{R}^{d-1}$, such that the polynomial 
$s^{d-1}+\lambda_{d-1}s^{d-2}+\ldots+\lambda_1$ with $s\in\mathbb{C}$ is Hurwitz 
\citep{Hurwitz1895}. The application of this policy leads to the error dynamics \citep{Chowdhary2015}
\begin{align}
\label{eq:err_dyn}
\dot{\bm{e}}=\underbrace{\begin{bmatrix}
0&1&0&\ldots&0\\
0&0&1&\ldots&0\\
\vdots&\vdots&\vdots&\ddots&\vdots\\
-k_c\lambda_1&-k_c\lambda_2&-k_c\lambda_3&\ldots&-k_c
\end{bmatrix}}_{\bm{A}}\bm{e} +\begin{bmatrix}
0\\0\\ \vdots\\f(\bm{x})-\nu_N(\bm{x})
\end{bmatrix}.
\end{align}
The first addend corresponds to the nominal error dynamics, which are described 
by a linear system with dynamics matrix $\bm{A}$, while the second addend 
represents the effect of the learning error on the tracking error dynamics. 

\begin{remark}
	Our assumption on the available training data allows noisy observations of the 
	derived state $\dot{x}_d$ while the states $\bm{x}$ themselves must be measured
	noise free. Although this assumption is debatable, it reflects practical 
	implementation well due to the fact that the time derivative $\dot{x}_d$ is 
	typically realized with finite difference approximations. These approximations
	inject considerably more noise than direct measurements of the state $\bm{x}$.
\end{remark}

\subsection{Probabilistic Stability Analysis}
\label{subsec:stab}
For safety critical applications, e.g., when robots and human work in 
close proximity, it is crucial to formally verify that the controlled
system satisfies safety constraints. This requires a sufficiently 
precise model $\nu_N(\cdot)$ and a properly chosen control gain $k_c$ 
as well as filter coefficients $\bm{\lambda}$ such that an upper bound
on the tracking error $\bm{e}$ can be determined as defined in the
following.
\begin{definition}[Ultimate Boundedness]
	\label{def:boundedness}
	The tracking error~$\bm{e}(t)$ between a dynamical system~$\dot{\bm{x}} = \bm{f}(\bm{x},\bm{u})$ 
	and a reference trajectory $\bm{x}_{\mathrm{ref}}$ 
	is ultimately bounded, if there exists a positive constant~$b\in 
	\mathbb{R}_+$ such that for every~$a\in \mathbb{R}_+$, there is a~$T=T(a,b)\in 
	\mathbb{R}_+$ such that
	\begin{align*}
	\|\bm{e}(t_0)\|\leq a \quad \Rightarrow \quad \|\bm{e}(t)\|\leq b, 
	\forall t\geq t_0 + T.
	\end{align*}
\end{definition}
As an analytical computation of the trajectories~$\bm{x}(t)$ is generally
impossible, we conduct a stability analysis based on Lyapunov theory. 
This analysis exhibits the advantage that conclusions about the closed-loop
system behavior can be drawn without the necessity of simulating the 
controlled system or even executing the policy on the real system~\citep{Khalil2002}.
In the following, we consider a bounded reference trajectory $\bm{x}_{\mathrm{ref}}$
in order to 
comply with the requirement of a compact state space $\mathbb{X}$ for the 
uniform error bound of Gaussian process regression.
\begin{lemma}[\citep{Khalil2002}]
	\label{lem:Lyap}
	Consider a dynamical system~$\dot{\bm{x}} = \bm{f}(\bm{x},\bm{u})$
	and a bounded reference trajectory~$\bm{x}_{\mathrm{ref}}$. If there 
	exists a positive definite (so called Lyapunov) function,~\mbox{$V:\mathbb{R}^d \to \mathbb{R}_{+,0}$},
	and a set $\mathbb{V}=\{\bm{e}\in\mathbb{R}^d| V(\bm{e})\leq\ubar{v}\}$ for 
	$\ubar{v}\in\mathbb{R}^+$, such that
	\begin{align*}
		\dot{V}(\bm{e})&< 0 &\forall \bm{e}\in\{\bm{e}\in\mathbb{R}^d|  \ubar{v}<V(\bm{e})\leq \bar{v}\}\\
		\ubar{v} &\leq \bar{v}=\min\limits_{\bm{x}\in\mathbb{R}^d\setminus\mathbb{X}}V(\bm{x}-\bm{x}_{\mathrm{ref}}),
	\end{align*}
	then the tracking error $\bm{e}$ is ultimately bounded to the set $\mathbb{V}$ for all $\bm{x}(t_0)\in\{\bm{x}\in\mathbb{X}| V(\bm{x}-\bm{x}_{\mathrm{ref}})<\bar{v}\}$.
\end{lemma}
Based on this lemma we can check the satisfaction of safety constraints 
by explicitly determining the set $\mathbb{V}$ as shown in the following.
\begin{theorem}
	\label{thm:Stable}
	Consider a control affine system~\eqref{eq:sys}, where~$f(\cdot)$ 
	admits a Lipschitz constant $L_f$ on $\mathbb{X}\subset\mathbb{R}^d$, 
	with feedback linearizing controller~\eqref{eq:FeliCtrl}. Let $\bm{P}=\begin{bmatrix}
	\bm{p}_1&\ldots&\bm{p}_d
	\end{bmatrix}\in\mathbb{R}^{d\times d}$
	the unique, positive definite solution to the algebraic Riccati equation 
	\begin{align*}
	\bm{A}^T\bm{P}+\bm{P}\bm{A}=-\bm{I}_d
	\end{align*} 
	with $\bm{A}$ defined in \eqref{eq:err_dyn}. 
	Assume that~$f(\cdot)$ satisfies \cref{ass:samplefun} and the observations 
	$y^{(i)}$, $i=1,\ldots,N$, satisfy the conditions of \cref{ass:noise}. Consider 
	the set of initial states $\mathbb{X}_0$ and the Lyapunov decrease region $\mathbb{L}$
	defined as
	\begin{align*}
		\mathbb{X}_0&=\left\{\bm{x}\in\mathbb{X}\bigg| (\bm{x}-\bm{x}_{\mathrm{ref}})^T\bm{P}(\bm{x}-\bm{x}_{\mathrm{ref}})\leq \min\limits_{\bm{x}'\in\mathbb{R}^d\setminus\mathbb{X}}(\bm{x}'-\bm{x}_{\mathrm{ref}})^T\bm{P}(\bm{x}'-\bm{x}_{\mathrm{ref}}) \right\}\\
		\mathbb{L}&=\left\{ \bm{x}\in\mathbb{X}\big|  \|\bm{x}-\bm{x}_{\mathrm{ref}}\|^2\leq 2\left(\sqrt{\beta(\tau)}\sigma_N(\bm{x})+\gamma(\tau)\right)\left|(\bm{x}-\bm{x}_{\mathrm{ref}})^T\bm{p}_d\right| \right\}
	\end{align*}
	with $\beta(\tau)$ and $\gamma(\tau)$ from \cref{th:errbound_with}. 
	If $\mathbb{X}\setminus\mathbb{L}\subseteq \mathbb{X}_0$, then the tracking 
	error $\bm{e}(t)$ converges with probability of at least $1-\delta$ 
	for all $\bm{x}(t_0)\in\mathbb{X}_0$ to the ultimately bounded set
	\begin{align*}
		\mathbb{V}=\left\{\bm{e}\in\mathbb{R}^d\big| \bm{e}^T\bm{P}\bm{e}\leq\ubar{v}\right\}
	\end{align*}
	with ultimate bound
	\begin{align*}
	\ubar{v}=\max\limits_{\bm{x}\in\mathbb{L}}(\bm{x}-\bm{x}_{\mathrm{ref}})^T\bm{P}(\bm{x}-\bm{x}_{\mathrm{ref}}).
	\end{align*}
\end{theorem}
It is trivial to see that the ultimate bound $\ubar{v}$ and consequently the 
extension of the ultimately bounded set $\mathbb{V}$ can be made 
arbitrarily small by reducing the posterior variance $\sigma_N(\cdot)$, 
which can be achieved by adding more training data. Therefore, it is 
straight forward to prove a vanishing control error $\bm{e}$ in the 
limit of infinite training data, which is defined as asymptotic stability 
in control literature~\citep{Khalil2002}. This is shown in the following
corollary.

\begin{corollary}
	\label{th:asstab}
	Consider a control affine system~\eqref{eq:sys}, where~$f(\cdot)$ satisfies 
	\cref{ass:samplefun} for a Gaussian process defined through the covariance
	kernel $k(\cdot,\cdot)$ with continuous partial derivatives up to the fourth
	order on the set $\mathbb{X}$ and the infinite observation sequence 
	$y^{(i)}$, $i=1,\ldots,\infty$, satisfies 
	the conditions of \cref{ass:noise}. If the assumptions of \cref{th:varvan}
	are satisfied with $\alpha(N)\in\mathcal{O}\left(\log(N)^{-1-\epsilon}\right)$ for all $\bm{x}\in\mathbb{X}$, 
	then the tracking error satisfies the ultimate bound ultimate bound $\ubar{v}\in\mathcal{O}\left( \sqrt{\log(N)/\alpha(N)} \right)$ with probability of at least $1-\delta$. Moreover, the 
	feedback linearizing controller almost surely asymptotically 
	stabilizes the system for all initial states $\bm{x}(t_0)\in\mathbb{X}_0$ in the limit of 
	infinitely many training samples.
\end{corollary}

\section{Numerical Evaluation}
\label{sec:numEval}
In this section we illustrate the behavior 
of the proposed bounds for the posterior variance,
the learning error and the control error. \cref{subsec:VarUniform} 
compares our variance bounds to the exact posterior
variance and bounds from literature for different sampling
distributions. In \cref{subsec:simPrior} we illustrate the
importance of suitable prior distributions for meaningful
uniform error bounds and evaluate the proposed constrained
hyperparameter optimization. Finally, the data-dependency of 
the feedback linearizing controller is analyzed in \cref{subsec:safety}
before it is applied to reference tracking 
with a real-world robotic manipulator in \cref{subsec:robot}.

\subsection{Decrease Rate of Posterior Variance Bounds}
\label{subsec:VarUniform}

We compare the bounds in \cref{th:var_bound} and 
Corollary~\ref{cor:var_bound} to the exact 
posterior variance for GPs with a 
squared exponential, a Mat\'ern kernel with 
$\nu=\frac{1}{2}$, a polynomial kernel with $p=3$ and 
a neural network kernel. Furthermore, we evaluate the 
bound $\sigma_2^2(\bm{x})$ proposed in \citep{Williams2000}, which 
considers the two closest points in the data set 
to the test point resulting in 
\begin{align*}
\bar{\sigma}_2^2(x)=
k(0)-\frac{(k(0)+\sigma_n^2)(k^2(\rho_2)+k^2(\rho_1))-2k(\eta)k(\rho_1)k(\rho_2)}
{(k(0)+\sigma_n^2)^2-k^2(\eta)},
\end{align*}
where~$\rho_1$ and~$\rho_2$ are the distances to the two closest training 
samples and $\eta$ is the distance between the two closest training samples.
Additionally, we consider the bound for the mean square prediction error proposed
in~\citep{Wang2018} which corresponds to a bound on the posterior variance. This bound
is given by
\begin{align*}
	\bar{\sigma}_{\mathrm{MSPE}}^2(\bm{x})&=k(\bm{x},\bm{x})-2k(\underline{\bm{x}},\bm{x})+
	k(\underline{\bm{x}},\underline{\bm{x}})-\frac{(k(\underline{\bm{x}},\bm{x})+
		k(\underline{\bm{x}},\underline{\bm{x}}))^2}{N\sup\limits_{\bm{x}',\bm{x}''\in\mathbb{X}}
		k(\bm{x}',\bm{x}'')+\sigma_n^2}\nonumber\\
	&+\frac{\sigma_n^2(N\sup\limits_{\bm{x}',\bm{x}''\in\mathbb{X}}
		k(\bm{x}',\bm{x}'')+2(k(\underline{\bm{x}},\bm{x})+
		k(\underline{\bm{x}},\underline{\bm{x}}))^2)}{N\sup\limits_{\bm{x}',\bm{x}''\in\mathbb{X}}
		k(\bm{x}',\bm{x}'')+\sigma_n^2},
\end{align*}
where $\underline{\bm{x}}$ denotes the closest point in the training data set 
to the test point $\bm{x}$. The posterior variance is 
evaluated at the point $x=1$ for a uniform training 
data distribution $\mathcal{U}([0.5,1.5])$. 
Furthermore, the 
length scale of the kernels 
is set to $l=1$ where applicable and the noise 
variance is set to~$\sigma_n^2=0.1$. 
In order to obtain a good value 
for the information radius $\rho$, consider the 
following approximation of \eqref{eq:isobound} for 
isotropic kernels 
\begin{align*}
\hat{\sigma}_{\rho}^2(1)\approx k(0)-
\frac{k^2(\rho)}{k(0)}+\frac{k(0)\sigma_n^2}{N\rho k(0)
	+\sigma_n^2},
\end{align*}
where we use the expectation of 
$E[|\mathbb{B}_{\rho}(1)|]=N\rho$ instead of the 
random variable~$|\mathbb{B}_{\rho}(1)|$. For the 
squared exponential kernel the Taylor 
expansion around~$\rho=0$ yields
\begin{align*}
k(0)-\frac{k^2(\rho)}{k(0)}\approx 
2\frac{\rho^2}{l^2}+\mathcal{O}(\rho^3).
\end{align*}
Therefore, for large~$N$ the best asymptotic behavior 
of \eqref{eq:isobound} 
is achieved with~$\rho(N)=cN^{-\frac{1}{3}}$ 
for the squared exponential kernel under uniform sampling
and leads 
to~$\hat{\sigma}_{\rho}^2(1)\approx \mathcal{O}
(N^{-\frac{2}{3}})$.
The same approach can 
be used to calculate the information radius~$\rho(N)$ with 
the best asymptotic behavior of the bound in
\cref{cor:var_bound} for the Mat\'ern kernel
with~$\nu=\frac{1}{2}$. This leads
to~$\rho(N)=cN^{-\frac{1}{2}}$ and an asymptotic 
behavior of~$\hat{\sigma}_{\rho}^2(1)\approx 
\mathcal{O}(N^{-\frac{1}{2}})$. For the 
non-isotropic kernels, we pursue a similar 
approach and substitute the expected number of
samples~$N\rho$ in \eqref{eq:sigbound}, 
which results in the asymptotically
optimal~$\rho(N)=cN^{-\frac{1}{2}}$ and 
$\hat{\sigma}_{\rho}^2(1)\approx 
\mathcal{O}(N^{-\frac{1}{2}})$. For the general bound in 
Theorem~\ref{th:var_bound} we consider $\bm{x}^*=\bm{x}=1$ as
reference point for the isotropic kernels, whereas we choose 
$\bm{x}^*=1.4$ as reference point for the non-isotropic kernels 
because it exhibits higher variance $k(1.4,1.4)$ and lies not
on the boundary of the considered interval~$[0.5,1.5]$. The 
posterior variance bounds and the exact posterior variance 
$\sigma_N^2(1)$ averaged over $20$ training data sets are 
illustrated in \cref{fig:varUni}.

\begin{figure}[t]
	\centering
		\begin{minipage}{0.47\textwidth}
			\centering
			\def\file{plots/var_uniform_squaredexponential.txt}
			\tikzsetnextfilename{var_SE}
			\begin{tikzpicture}
			\begin{loglogaxis}[xmin=0,xmax=12200,ymin=0.000004,ymax=0.99, samples=100,
			grid=none, axis y line=left, axis x line=bottom, 
			ylabel={$\sigma_N^2$}, 
			scaled x ticks=false,legend pos=south west, title={Squared Exponential}
			]
			\addplot[black] table[x = idx,y = sig_m ]{\file};
			\addplot[red,dashed] table[x = idx,y = sig_bm ]{\file};
			\addplot[red] table[x = idx,y = sig_bm_gen ]{\file};
			\addplot[blue] table[x = idx,y = sig2 ]{\file};
			\addplot[green] table[x = idx,y = sig_wang ]{\file};
			\end{loglogaxis}
			\end{tikzpicture}
		\end{minipage}\hfill
		\begin{minipage}{0.47\textwidth}
			\centering
			\def\file{plots/var_uniform_exponential.txt}
			\tikzsetnextfilename{var_exp}
			\begin{tikzpicture}
			\begin{loglogaxis}[xmin=0,xmax=12200,ymin=0.000004,ymax=0.99, samples=100,
			grid=none, axis y line=left, axis x line=bottom, 
			ylabel={$\sigma_N^2$}, 
			scaled x ticks=false, 
			legend columns = 2, 
			legend style={at={(-0.215,-0.55)}, anchor= west}, title={Exponential}
			]
			\addplot[black] table[x = idx,y = sig_m ]{\file};
			\addplot[red,dashed] table[x = idx,y = sig_bm ]{\file};
			\addplot[red] table[x = idx,y = sig_bm_gen ]{\file};
			\addplot[blue] table[x = idx,y = sig2 ]{\file};
			\addplot[green] table[x = idx,y = sig_wang ]{\file};
			\end{loglogaxis}
			\end{tikzpicture}
		\end{minipage}
		\begin{minipage}{0.47\textwidth}
			\centering
			\def\file{plots/var_uniform_polynomial.txt}
			\tikzsetnextfilename{var_poly}
			\begin{tikzpicture}
			\begin{loglogaxis}[xmin=0,xmax=12200,ymin=0.000004,ymax=0.99, samples=100,
			grid=none, axis y line=left, axis x line=bottom, 
			xlabel=$N$,ylabel={$\sigma_N^2$}, 
			scaled x ticks=false,legend pos=south west, title={Polynomial}
			]
			\addplot[black] table[x = idx,y = sig_m ]{\file};
			\addplot[red] table[x = idx,y = sig_bm_gen ]{\file}; 
			\addplot[green] table[x = idx,y = sig_wang ]{\file};
			\end{loglogaxis}
			\end{tikzpicture}
		\end{minipage}\hfill
		\begin{minipage}{0.47\textwidth}
			\centering
			\def\file{plots/var_uniform_neuralnetwork.txt}
			\tikzsetnextfilename{var_nn}
			\begin{tikzpicture}
			\begin{loglogaxis}[xmin=0,xmax=12200,ymin=0.000004,ymax=0.99, 
			samples=100,
			grid=none, axis y line=left, axis x line=bottom, 
			xlabel=$N$,ylabel={$\sigma_N^2$}, 
			scaled x ticks=false, 
			legend columns = 2, 
			legend style={at={(-0.215,-0.55)}, anchor= west}, title={Neural 
			Network}
			]
			\addplot[black] table[x = idx,y = sig_m ]{\file};
			\addplot[red] table[x = idx,y = sig_bm_gen ]{\file};
			\addplot[green] table[x = idx,y = sig_wang ]{\file};
			\end{loglogaxis}
			\end{tikzpicture}
		\end{minipage}\\
		\begin{minipage}{\columnwidth}
			\pgfplotsset{width=30\columnwidth/100,
				height = 30\columnwidth/100 }
			\centering \tikzset{external/export next=false}
			\begin{tikzpicture} 
			\begin{axis}[%
			hide axis,
			xmin=10,
			xmax=50,
			ymin=0,
			ymax=0.4,
			legend columns = 3,
			legend style={draw=white!15!black,legend cell align=left}
			]
			\addlegendimage{black}
			\addlegendentry{$\sigma^2_{\mathrm{num}}(1)$};
			\addlegendimage{red,dashed}
			\addlegendentry{$\hat{\sigma}_{\rho}^2(1)$ from 
			Cor.~\ref{cor:var_bound}};
			\addlegendimage{red}
			\addlegendentry{$\bar{\sigma}_{\rho}^2(1)$ from 
			Thm.~\ref{th:var_bound}};
			\addlegendimage{blue}
			\addlegendentry{$\bar{\sigma}_{2}^2(1)$ \citep{Williams2000}};
			\addlegendimage{green}
			\addlegendentry{$\bar{\sigma}_{\mathrm{MSPE}}^2(1)$ 
			\citep{Wang2018}};
			\end{axis}
			\end{tikzpicture}
		\end{minipage}

	\vspace{-0.25cm}	
	\caption{Average posterior variance and bounds 
		of the squared exponential, the Mat\'ern
		kernel with~$\nu=\frac{1}{2}$, the 
		polynomial kernel with~$p=3$ and the 
		neural network kernel for 
		uniformly sampled training data}
	\label{fig:varUni}
\end{figure}
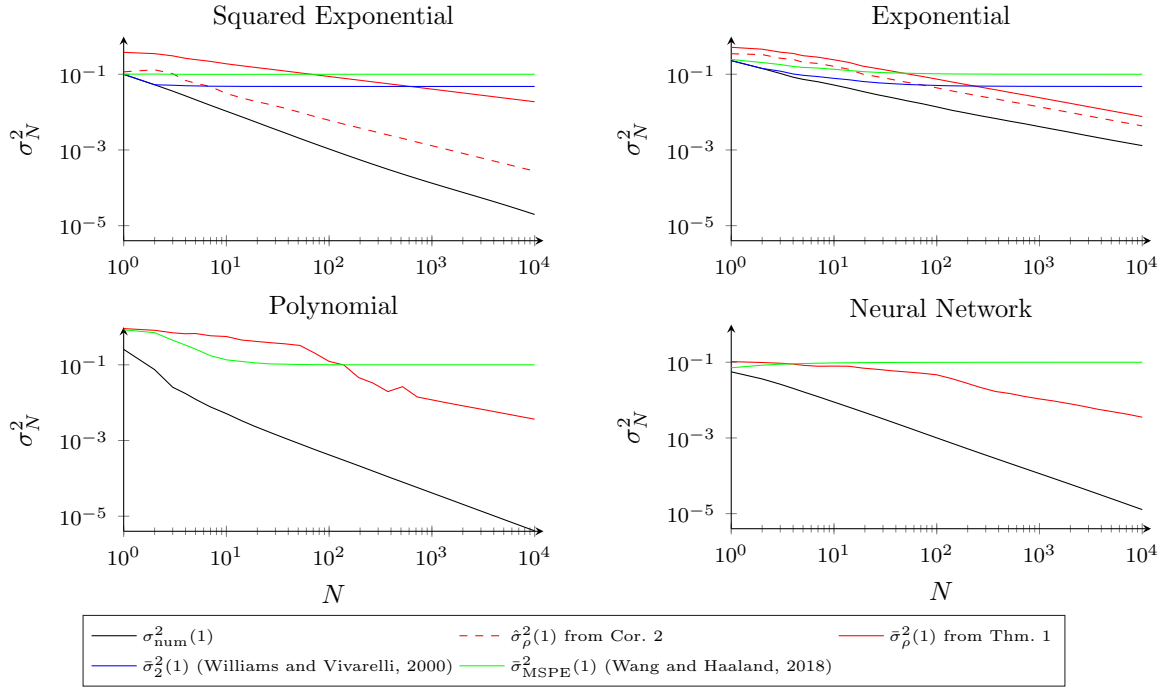

\begin{figure}[t]
			\center
		\begin{minipage}{0.47\textwidth}
			\centering
			\def\file{plots/var_vanishing_squaredexponential.txt}
			\tikzsetnextfilename{var_SE_vanishing}
			\begin{tikzpicture}
			\begin{loglogaxis}[xmin=0,xmax=12200,ymin=0.000004,
			ymax=0.99, samples=100,
			grid=none, axis y line=left, axis x line=bottom, 
			ylabel={$\sigma_N^2$}, 
			scaled x ticks=false,legend pos=south west, title={Squared Exponential}
			]
			\addplot[black] table[x = idx,y = sig_m ]{\file};
			\addplot[red,dashed] table[x = idx,y = sig_bm ]{\file};
			\addplot[red] table[x = idx,y = sig_bm_gen ]{\file};
			\addplot[blue] table[x = idx,y = sig2 ]{\file};
			\addplot[green] table[x = idx,y = sig_wang ]{\file};
			\end{loglogaxis}
			\end{tikzpicture}
		\end{minipage}\hfill
		\begin{minipage}{0.47\textwidth}
			\centering
			\def\file{plots/var_vanishing_exponential.txt}
			\tikzsetnextfilename{var_exp_vanishing}
			\begin{tikzpicture}
			\begin{loglogaxis}[xmin=0,xmax=12200,ymin=0.000004,
			ymax=0.99, samples=100, grid=none, axis y line=left, 
			axis x line=bottom, ylabel={$\sigma_N^2$},
			scaled x ticks=false, 
			legend columns = 2, 
			legend style={at={(-0.215,-0.55)}, anchor= west}, title={Exponential}
			]
			\addplot[black] table[x = idx,y = sig_m ]{\file};
			\addplot[red,dashed] table[x = idx,y = sig_bm ]{\file};
			\addplot[red] table[x = idx,y = sig_bm_gen ]{\file};
			\addplot[blue] table[x = idx,y = sig2 ]{\file};
			\addplot[green] table[x = idx,y = sig_wang ]{\file};
			\end{loglogaxis}
			\end{tikzpicture}
		\end{minipage}
		\begin{minipage}{0.47\textwidth}
			\centering
			\def\file{plots/var_vanishing_polynomial.txt}
			\tikzsetnextfilename{var_poly}
			\begin{tikzpicture}
			\begin{loglogaxis}[xmin=0,xmax=12200,ymin=0.000004,ymax=0.99, samples=100,
			grid=none, axis y line=left, axis x line=bottom, 
			xlabel=$N$,ylabel={$\sigma_N^2$}, 
			scaled x ticks=false,legend pos=south west, title={Polynomial}
			]
			\addplot[black] table[x = idx,y = sig_m ]{\file};
			\addplot[red] table[x = idx,y = sig_bm_gen ]{\file};
			\addplot[green] table[x = idx,y = sig_wang ]{\file};
			\end{loglogaxis}
			\end{tikzpicture}
		\end{minipage}\hfill
		\begin{minipage}{0.47\textwidth}
			\centering
			\def\file{plots/var_vanishing_neuralnetwork.txt}
			\tikzsetnextfilename{var_nn}
			\begin{tikzpicture}
			\begin{loglogaxis}[xmin=0,xmax=12200,ymin=0.000004,ymax=0.99, samples=100,
			grid=none, axis y line=left, axis x line=bottom, 
			xlabel=$N$,ylabel={$\sigma_N^2$}, 
			scaled x ticks=false, 
			legend columns = 2, 
			legend style={at={(-0.215,-0.55)}, anchor= west}, title={Neural Network}
			]
			\addplot[black] table[x = idx,y = sig_m ]{\file};
			\addplot[red] table[x = idx,y = sig_bm_gen ]{\file};
			\addplot[green] table[x = idx,y = sig_wang ]{\file};
			\end{loglogaxis}
			\end{tikzpicture}
		\end{minipage}\\
		\begin{minipage}{\textwidth}
		\pgfplotsset{width=30\columnwidth/100,
			height = 30\columnwidth/100 }
		\centering	\tikzset{external/export next=false}
		\begin{tikzpicture} 
		\begin{axis}[%
		hide axis,
		xmin=10,
		xmax=50,
		ymin=0,
		ymax=0.4,
		legend columns = 3,
		legend style={draw=white!15!black,legend cell align=left}
		]
		\addlegendimage{black}
		\addlegendentry{$\sigma^2_{\mathrm{num}}(1)$};
		\addlegendimage{red,dashed}
		\addlegendentry{$\hat{\sigma}_{\rho}^2(1)$ from 
		Cor.~\ref{cor:var_bound}};
		\addlegendimage{red}
		\addlegendentry{$\bar{\sigma}_{\rho}^2(1)$ from 
		Thm.~\ref{th:var_bound}};
		\addlegendimage{blue}
		\addlegendentry{$\bar{\sigma}_{2}^2(1)$ \citep{Williams2000}};
		\addlegendimage{green}
		\addlegendentry{$\bar{\sigma}_{\mathrm{MSPE}}^2(1)$ \citep{Wang2018}};
		\end{axis}
		\end{tikzpicture}
	\end{minipage}				
		
	\vspace{-0.25cm}
	\caption{Average posterior variance and bounds 
		of the squared exponential, the Mat\'ern
		kernel with~$\nu=\frac{1}{2}$, the 
		polynomial kernel with~$p=3$ and the 
		neural network kernel for 
		training data sampled from vanishing distribution}
	\label{fig:VarVan}
\end{figure}
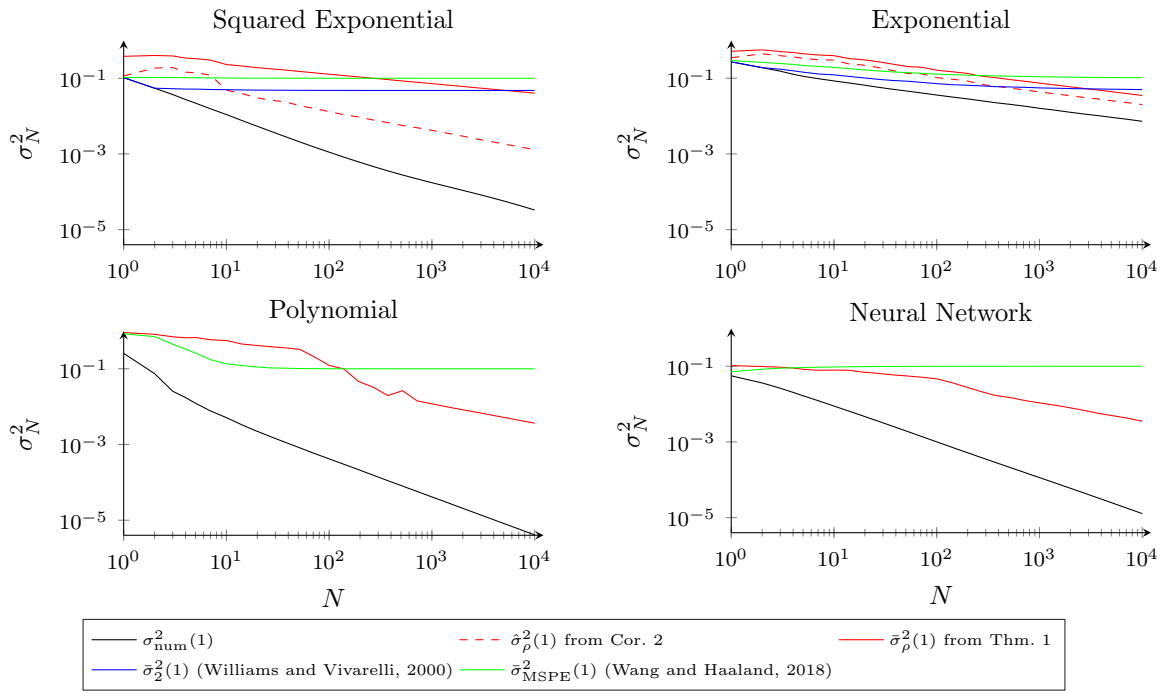

We also compare the bounds in \cref{th:var_bound} and 
\cref{cor:var_bound} to the exact 
posterior variance for training data sampled from 
the distribution with density function
\begin{align*}
p(x)=4|1-x|,\quad 0.5\leq x\leq 1.5.
\end{align*}
This probability density vanishes at the test point 
$x=1$ and it leads to~$\tilde{p}(N)=4\rho^2(N)$
for~\mbox{$\rho(N)\leq 0.5$}. By employing a Taylor 
expansion of the kernel around the test point,
we can derive the optimal asymptotic decay rates
for~$\rho(N)$ as in the previous section. 
For the Mat\'ern kernel, this leads to 
$\rho(N)=cN^{-\frac{1}{3}}$ and an asymptotic 
behavior of the posterior variance~$\sigma_N^2(1)\approx
\mathcal{O}(N^{-\frac{1}{3}})$. For the squared 
exponential kernel, a slightly faster decreasing 
$\rho(N)=cN^{-\frac{1}{4}}$ can be chosen, which 
results in~$\hat{\sigma}_{\rho}^2(1)\approx 
\mathcal{O}(N^{-\frac{1}{2}})$. For the non-isotropic
kernels we choose the reference point at $\bm{x}^*=1.4$
such that it does not suffer from the vanishing 
distribution. Hence, we choose $\rho(N)=cN^{-\frac{1}{2}}$
as for the uniform distributed training data. 
The resulting posterior variance bounds and the 
exact posterior variance are illustrated 
in \cref{fig:VarVan}. \looseness=-1

The posterior variance bounds for the isotropic
squared exponential and Mat\'ern kernel exhibit 
a similar decrease rate as the actually observed 
one in \cref{fig:varUni} and \cref{fig:VarVan}. 
Indeed, the bound for the
Mat\'ern kernel shows the exact same behavior
and only differs by a constant factor for large~$N$. 
However, for non-isotropic kernels, our 
bound in \cref{th:var_bound} is rather loose 
as it converges with~$\mathcal{O}(N^{-\frac{1}{2}})$ 
while the true posterior variance exhibits a decay
rate of approximately~$\mathcal{O}(N^{-1})$ for the 
uniform distribution in \cref{fig:varUni}. Furthermore, 
merely a small difference of the decrease rate of the 
numerically estimated posterior variance and our 
proposed bound $\bar{\sigma}_{\rho}^2(1)$ can be observed 
between both figures. This is caused by the non-isotropy
of the kernel and the bound: the kernel considers 
data globally, while our bound considers data locally 
around a non-local reference point $\bm{x}^*$. Finally,
both of our proposed bounds are advantageous compared 
to existing approaches in the fact that they converge to
a constant decrease rate, whereas the existing bounds 
converge to a constant value. Therefore, our bounds 
can be applied to for small as well as large numbers 
of training samples, while existing approaches work 
well only for small numbers of training samples.

\subsection{Importance of Priors for Uniform Error Bounds}
\label{subsec:simPrior}

It is always possible to find a probability $\delta$ such that the uniform
error bound \eqref{eq:un_err} holds on the whole considered set $\mathbb{X}$.
However, this value can be very small. Therefore, suitable hyperparameters 
are crucial to shape the prior distribution such that it assigns proper 
probability to the functions that coincide with the prior information and 
consequently, the uniform error bound holds with reasonably small values 
of $\delta$. In order to demonstrate the effectiveness of incorporating 
prior information about maximum values and maximum derivative
values through constraints on the hyperparameters, we compare 
the satisfaction of the uniform error bound \eqref{eq:errorbound} 
for constrained and unconstrained hyperparameters on two practically
relevant scenarios.

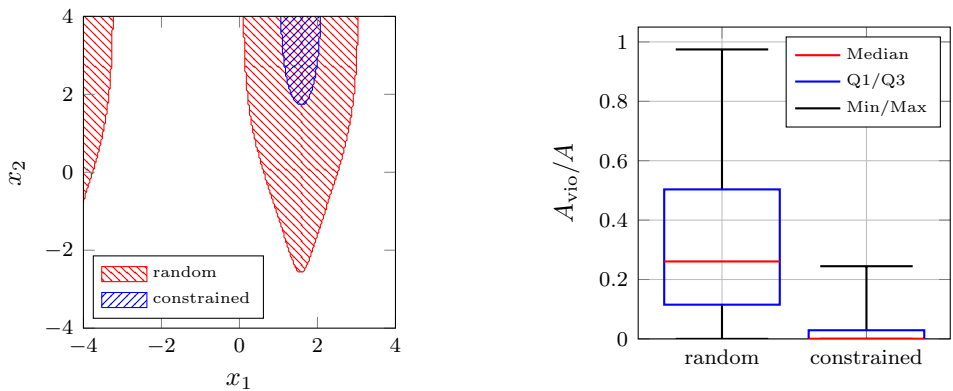
\begin{figure}
	\pgfplotsset{width=80\columnwidth/100,
		height = 80\columnwidth/100 }
	\centering
	\begin{minipage}{0.47\textwidth}
		\tikzsetnextfilename{SEisobound1}
		\def\file{plots/SEisobound1.txt}
		\center
		\begin{tikzpicture}
		\begin{axis}[grid=none,enlargelimits=false, axis on top,
		xlabel={$x_1$}, ylabel={$x_2$}, legend pos = south east,
		xmin=-4, xmax = 4, ymin = -4, ymax = 4,	legend columns=1,
		legend pos=south west,
		axis background/.style={fill=white},
		]
		\addplot[red, fill=red, pattern=north west lines,pattern color=red,forget plot] table[x=contred0_c2_1,y=contred0_c2_2]{\file};
		\addplot[red, fill=red,pattern=north west lines,pattern color=red,forget plot] table[x=contred0_c1_1,y=contred0_c1_2]{\file};
		\addplot[blue,fill=blue,pattern=north east lines, pattern color=blue,forget plot] table[x=contred_c1_1,y=contred_c1_2]{\file};
		
		\addlegendimage{area legend,red,fill=red,pattern=north west lines, pattern color=red}
		\addlegendentry{random};
		\addlegendimage{area legend,blue,fill=blue,pattern=north east lines, pattern color=blue}
		\addlegendentry{constrained};
		\end{axis}
		\end{tikzpicture}
	\end{minipage}\hfill
	\begin{minipage}{0.47\textwidth}
		\pgfplotsset{boxplot/every median/.style={color=red,thick},
			boxplot/every box/.style={blue=red,thick},
			boxplot/every whisker/.style={color=black,thick},}
		\begin{tikzpicture}
		\begin{axis}[
		ymax=1.05,
		ymin=0,
		boxplot/draw direction=y,
		xtick distance=1,
		ytick distance=0.2,
		ylabel={$A_{\mathrm{vio}}/A$},
		xticklabels={,,{random},{constrained}},xticklabel style={align=center},legend pos=north east]		]
		\addlegendimage{thick,red};
		\addlegendentry{Median};
		\addlegendimage{thick,blue};
		\addlegendentry{Q1/Q3};
		\addlegendimage{thick,black};
		\addlegendentry{Min/Max};
		
		] coordinates {};
		
		\addplot+[black,
		boxplot prepared={
			median=0.2608,
			upper quartile=0.5034,
			lower quartile=0.1152,
			upper whisker=0.9747,
			lower whisker=0.0
		},
		] coordinates {};
		\addplot+[black,
		boxplot prepared={
			median=0.0,
			upper quartile=0.0291,
			lower quartile=0.0,
			upper whisker=0.2447,
			lower whisker=0.0
		},
		] coordinates {};
		
		\end{axis}
		\end{tikzpicture}
	\end{minipage}
	\caption{Violation of the uniform error bound for GP regression with isotropic squared 
		exponential kernel, where the hyperparameters are obtained from constrained and 
		unconstrained random sampling}
	\label{fig:errorBound_isoprior}
\end{figure}

The first scenario considers the problem of choosing
hyperparameters without any data, which is an issue that is typically
avoided in applications such as sequential decision making by 
assuming randomly distributed measurements. In physical applications, 
however, there is typically no principled way for obtaining such 
initial data. Therefore, we assume random hyperparameters instead 
and enforce the constraints \eqref{eq:fcon} and \eqref{eq:Lcon} 
by redrawing samples. We evaluate the performance on the function
$f(\bm{x})=\sin(x_1)+\frac{1}{1+\mathrm{e}^{-x_2}}$ and consider the
prior information $\bar{f}=1.5$, $\ubar{f}=-0.8$, $\bar{f}^{\partial 1}=0.8$, 
$\ubar{f}^{\partial 1}=-0.8$,$\bar{f}^{\partial 2}=0.2$ and 
$\ubar{f}^{\partial 2}=0.0$ with probability $\delta_f=\delta_L=0.99$. 
We investigate the uniform error bound with $\delta=0.1$ and $\tau=1e-2$
on the region $[-4,4]^2$ for a Gaussian process with isotropic 
squared exponential kernel. The signal standard deviation $\sigma_f$ 
and the inverse length scale $1/l$ are drawn from the exponential 
distribution with mean $0.05$ and we perform $1000$ simulation 
iterations. The violation of the uniform error bound for a 
constrained and unconstrained hyperparameter random sample is depicted 
on the left side of \cref{fig:errorBound_isoprior}, while the right side 
illustrates the quotient of the violation surface $A_{\mathrm{vio}}$ and 
the total surface $A$. 
Since the assumed bounds $\bar{f}$ and $\bar{f}^{\partial i}$ are smaller
than the actual function properties and the constraints
\eqref{eq:fcon} and \eqref{eq:Lcon} are merely necessary conditions, the 
satisfaction of the uniform error bound cannot be guaranteed in general 
for the chosen value of $\delta$. However, the simulations clearly 
indicate that the constraints significantly reduce the constraint violation
compared to purely random hyperparameters.

In the second scenario, we address the problem of training data which
lies only on a small subset of the input domain such that the data 
only captures limited information of the true system. We consider the 
function $f(\bm{x})=(1-\tanh(100x_2))x_1$ and prior bounds 
$\bar{f}=8$, $\ubar{f}=-8$, $\bar{f}^{\partial 1}=2$, $\ubar{f}^{\partial 1}=0$, 
$\bar{f}^{\partial 2}=400$ and $\ubar{f}^{\partial 2}=-400$ 
with probability $\delta_f=\delta_L=0.5$. The error bound is 
investigated on $[-4,4]^2$ with $\delta=0.01$, $\tau=1e-8$ and a 
probabilistic Lipschitz constant based on \cref{th:Lip_f} with 
$\delta_L=0.01$ for a Mat\'ern kernel with $\nu=\frac{3}{2}$ which 
is trained using $20$ training samples uniformly distributed on the set 
$[-4,4]\times[0,4]$. The result of single simulation and the violation
percentage over $1000$ iterations of constrained and unconstrained 
hyperparameter optimization are depicted in \cref{fig:errorBound_Maternprior}.
It can be clearly seen that the constraints \eqref{eq:fcon} and \eqref{eq:Lcon}
significantly contribute to the interpretability and reliability of 
the uniform error bound~\eqref{eq:errorbound}. 

\begin{figure}
	\pgfplotsset{width=80\columnwidth/100,
		height = 80\columnwidth/100 }
	\centering
	\begin{minipage}{0.47\textwidth}
		\tikzsetnextfilename{Maternbound}
		\def\file{plots/Maternbound1.txt}
		\center
		\begin{tikzpicture}
		\begin{axis}[grid=none,enlargelimits=false, axis on top,
		xlabel={$x_1$}, ylabel={$x_2$}, legend pos = south east,
		xmin=-4, xmax = 4, ymin = -4, ymax = 4,	legend columns=1,
		legend pos=south east]
		\addplot[black, only marks] table[x=Xtrain_1,y=Xtrain_2]{\file};
		\addlegendentry{training data};
		\addplot[red, fill=red,pattern=north west lines,pattern color=red,forget plot] table[x=cont0cell_c1_1,y=cont0cell_c1_2]{\file};
		\addplot[red, fill=red,pattern=north west lines,pattern color=red,forget plot] table[x=cont0cell_c2_1,y=cont0cell_c2_2]{\file};
		
		\addplot[blue,fill=blue,pattern=north east lines, pattern color=blue,forget plot] table[x=contcell_c1_1,y=contcell_c1_2]{\file};
		\addplot[blue,fill=blue,pattern=north east lines, pattern color=blue,forget plot] table[x=contcell_c2_1,y=contcell_c2_2]{\file};
		
		\addlegendimage{area legend,red,fill=red,pattern=north west lines, pattern color=red}
		\addlegendentry{unconstrained};
		\addlegendimage{area legend,blue,fill=blue,pattern=north east lines, pattern color=blue}
		\addlegendentry{constrained};
		\end{axis}
		\end{tikzpicture}
	\end{minipage}\hfill
	\begin{minipage}{0.47\textwidth}
		\pgfplotsset{boxplot/every median/.style={color=red,thick},
			boxplot/every box/.style={blue=red,thick},
			boxplot/every whisker/.style={color=black,thick},}
		\begin{tikzpicture}
		\begin{axis}[
		ymax=0.58,
		ymin=-0.0,
		boxplot/draw direction=y,
		xtick distance=1,
		ytick distance=0.1,
		ylabel={$A_{\mathrm{vio}}/A$},
		xticklabels={,,{unconstrained},{constrained}},xticklabel style={align=center},legend pos=north east]
		\addlegendimage{thick,red};
		\addlegendentry{Median};
		\addlegendimage{thick,blue};
		\addlegendentry{Q1/Q3};
		\addlegendimage{thick,black};
		\addlegendentry{Min/Max};
		
		] coordinates {};
		
		\addplot+[black,
		boxplot prepared={
			median=0.3200,
			upper quartile=0.34,
			lower quartile=0.3,
			upper whisker=0.4,
			lower whisker=0.0
		},
		] coordinates {};
		\addplot+[black,
		boxplot prepared={
			median=0.18,
			upper quartile=0.3,
			lower quartile=0.12,
			upper whisker=0.38,
			lower whisker=0.0
		},
		] coordinates {};
		
		\end{axis}
		\end{tikzpicture}
	\end{minipage}	
	\caption{Violation of the uniform error bound for GP regression with Mat\'ern kernel with 
		$\nu=\frac{3}{2}$, where the hyperparameters are obtained from constrained and 
		unconstrained loglikelihood optimization}
	\label{fig:errorBound_Maternprior}
\end{figure}
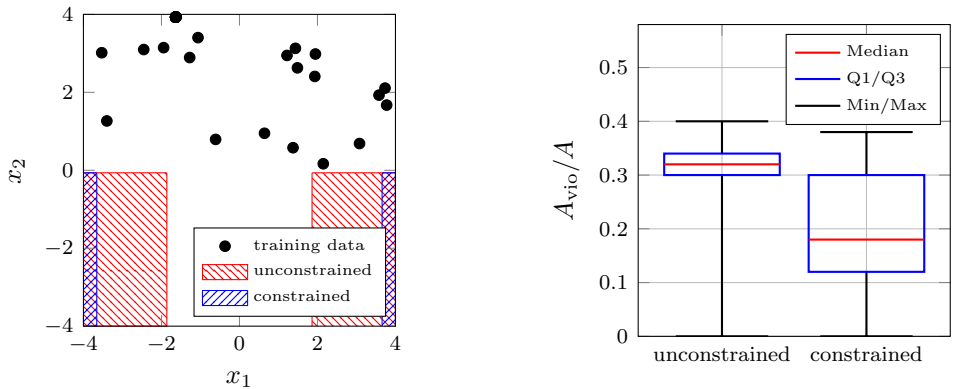

\subsection{Data-dependency of Safety Regions for Learning-based Control}
\label{subsec:safety}

Due to the standard deviation dependence of the tracking error bound 
in \cref{thm:Stable}, the guaranteed tracking performance is strongly influenced
by the training data distribution and density. We investigate these relationships on a 
two-dimensional system of the form~\eqref{eq:sys} 
with\linebreak~\mbox{$f(\bm{x}) = 1-\sin(2x_1) +  \frac{1}{1+\exp(-x_2)}$}. In 
order to demonstrate the effect of the distribution, we use a uniform 
grid over~$[0\ 3]\times[-5\ 5]$ with~$25$ points and~$\sigma_n^2 = 0.04$ as 
training data set, such that half of the considered state 
space~$\mathbb{X} =[-4\ 4]\times[-6\ 6]$ is not covered by training data.
The impact of the training data density is showed by determining the tracking error 
for data sets with $N=4m^2$, $m=1,\ldots,50$, samples on a uniform grid 
over $\mathbb{X}=[-3,3]\times[-5,5]$.
We consider a circular reference trajectory~$x_d(t) = 2\sin(t)$ 
and choose small gain controller gains~$k_c=4$ and~$\lambda=2$. A Gaussian
process with automatic relevance determination is employed for regression
and the hyperparameters are constrained with $\bar{f}=2.5$, 
$\bar{f}^{\partial 1}=1.6$ and $\bar{f}^{\partial 2}=0.2$ with probability 
$\delta_f=\delta_L=0.99$ according to~\eqref{eq:fbound} and \eqref{eq:Lfbound}.
The Lipschitz constant $L_f$ is determined probabilistically using 
Corollary~\ref{th:Lip_f} with $\delta_L=0.01$ leading to a conservative value
which is compensated by $\tau=10^{-6}$ in \cref{th:errbound_with}. We 
combine Corollary~\ref{th:Lip_f} and \cref{th:errbound_with} with 
$\delta=0.01$ using a union bound approximation. 

\begin{figure*}[t]
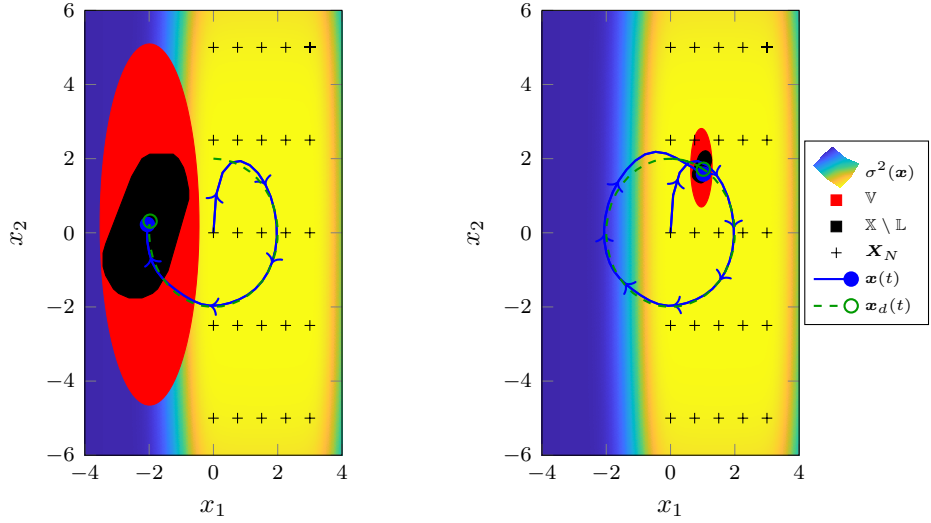

	\pgfplotsset{width=5cm,
		height = 7.5cm}
	\begin{minipage}{0.47\textwidth}
		\centering 
		\def\file{plots/2D_v2.txt}	 \def\imax{34} \def\ifin{200}
		\tikzsetnextfilename{LyapDecrTe_max}
		\begin{tikzpicture}
		\begin{axis}[view={0}{90},grid=none,enlargelimits=false, axis on top,
		xlabel={$x_1$}, ylabel={$x_2$}, legend pos = north west,
		xmin=-4, xmax = 4, ymin = -6, ymax = 6,	legend columns=5,
		legend style={at={(0,1.02)},anchor=south west}
		]
		\input{plots/2D_v2_GPvar_surf.tex}
		\addplot[red,fill=red, forget plot] table[x=Xhullmax_1,y=Xhullmax_2]{\file};
		\addplot[black,fill=black, forget plot] table[x=Linc_hull_max_1,y=Linc_hull_max_2]{\file};
		\addplot[only marks,black, mark=+]	table[x = Xtr_1, y = Xtr_2] {\file};
		\addplot[blue,-*,thick, forget plot]	
		table[x=Xsim_1,y=Xsim_2,skip coords between 
		index={\imax}{\ifin}]{\file};
		\addplot[blue,thick,forget plot, postaction={decorate,decoration={markings,
				mark=between positions 0.1 and 0.9 step 0.2 with {\arrow{>}}}}]	
		table[x=Xsim_1,y=Xsim_2,skip coords between 
		index={\imax}{\ifin}]{\file};
		\addplot[dashed,green!60!black,-o, thick]	
		table[x=Xd_1,y=Xd_2,skip coords between index={\imax}{\ifin}] {\file};
		\end{axis}
		\end{tikzpicture}
	\end{minipage}
	\begin{minipage}{0.47\textwidth}
		\def\file{plots/2D_v2.txt}	\def\imin{47} \def\ifin{200}
		\tikzsetnextfilename{LyapDecrTe_min}
		\begin{tikzpicture}
		\begin{axis}[view={0}{90},grid=none,enlargelimits=false, axis on top, 
		legend columns=1,
		xlabel={$x_1$}, ylabel={$x_2$}, legend style={at={(1.02,0.5)},anchor=west},
		xmin=-4, xmax = 4, ymin = -6, ymax = 6,
		]
		\input{plots/2D_v2_GPvar_surf.tex}
		\addlegendentry{$\sigma^2(\bm{x})$};
		\addplot[red,fill=red,forget plot] table[x=Xhullmin_1,y=Xhullmin_2]{\file};
		\addlegendimage{only marks, mark=square*,color=red}
		\addlegendentry{ $\mathbb{V}$};
		\addplot[black,fill=black, forget plot] table[x=Linc_hull_min_1,y=Linc_hull_min_2]{\file};
		\addlegendimage{only marks, mark=square*,color=black}
		\addlegendentry{ $\mathbb{X}\setminus\mathbb{L}$};
		\addplot[only marks,black, mark=+]	table[x = Xtr_1, y = Xtr_2] {\file};
		\addlegendentry{$\bm{X}_N$};
		\addplot[blue,-*,thick]	
		table[x=Xsim_1,y=Xsim_2,skip coords between 
		index={\imin}{\ifin}]{\file};
		\addplot[forget plot,blue,thick,postaction={decorate,decoration={markings,
				mark=between positions 0.1 and 0.9 step 0.1 with {\arrow{>}}}}]	
		table[x=Xsim_1,y=Xsim_2,skip coords between 
		index={\imin}{\ifin}]{\file};
		\addlegendentry{$\bm{x}(t)$};
		\addplot[dashed,green!60!black,-o, thick]	
		table[x=Xd_1,y=Xd_2, skip coords between index={\imin}{\ifin}] {\file};
		\addlegendentry{$\bm{x}_d(t)$};
		\end{axis}
		\end{tikzpicture}
	\end{minipage}
	\caption{Reference trajectory and simulated trajectory together with the 
		high probability ultimately bounded set $\mathbb{V}$. Low posterior standard 
		deviation leads to significantly smaller ultimate bounds.}
	\label{fig:LyapDecrTe}
\end{figure*}

\begin{figure}[t]
	\center
	\begin{minipage}{0.48\textwidth}
		\centering 
		\def\file{plots/2D_v2.txt}	\tikzsetnextfilename{error}
		\begin{tikzpicture}
		\begin{semilogyaxis}[xlabel={$t$},ylabel={$\sqrt{\bm{e}^T\bm{P}\bm{e}}$},
		xmin=0, ymin = 0.02, xmax = 30,ymax=8,legend columns=2]
		\addplot[blue]	table[x = T, y  = norme]{\file};
		\addplot[red]	table[x = T, y  = vubar]{\file};
		\legend{$\sqrt{\bm{e}^T\bm{P}\bm{e}}$, $\sqrt{\ubar{v}}$}
		\end{semilogyaxis}
		\end{tikzpicture}
	\end{minipage}\hfill
	\begin{minipage}{0.48\textwidth}
		\centering 
		\def\file{plots/2D_density_test.txt}	\tikzsetnextfilename{error_asym}
		\begin{tikzpicture}
		\begin{semilogyaxis}[xlabel={$N$},
		xmin=0, ymin = 0.002, xmax = 9999,ymax=2,legend columns=3]
		\addplot[blue]	table[x = Ns, y  = err_sim]{\file};
		\addplot[red]	table[x = Ns, y  = err_bound]{\file};
		\addplot[green]	table[x = Ns, y  = err_O]{\file};
		\legend{$\sqrt{\bm{e}^T\bm{P}\bm{e}}$, $\sqrt{\ubar{v}}$,$\sqrt{\ubar{v}_{\mathrm{asym}}}$}
		\end{semilogyaxis}
		\end{tikzpicture}
	\end{minipage}
		
	\caption{Relationship between ultimately bounded set $\mathbb{V}$ and simulated tracking error. 
		Left: High tracking errors correlate with large ultimate bounds. Right: The tracking error 
		and ultimate bound asymptotically converge to $0$ with increasing training set sizes.}
	\label{fig:error}
\end{figure}
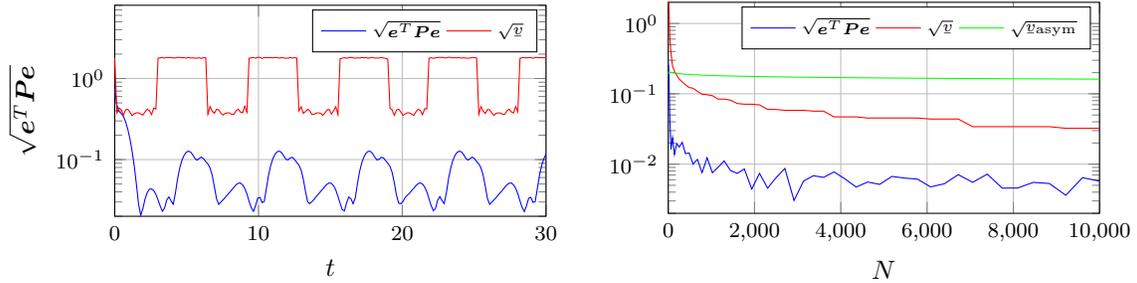

\cref{fig:LyapDecrTe,fig:error} depict the simulation results. It is clearly
illustrated that training data has a crucial impact on the posterior variance 
and thus, a lack of samples causes a large ultimately bounded set $\mathbb{V}$. 
It can be observed that the ultimate bound $\ubar{v}$ strongly correlates with the size 
of the tracking error observed in simulation as shown on the left side of 
\cref{fig:error}. However, the derived ultimate
bound is rather conservative which is a consequence of the small 
violation probability of~$1\%$. Moreover, the tracking error and the ultimate bound decrease 
with a growing number of training samples, as illustrated on the right side of \cref{fig:error}. 
In fact, it follows from Corollary~\ref{th:asstab} that $\ubar{v}\in\mathcal{O}(\ubar{v}_{\mathrm{asym}})$ 
with $\ubar{v}_{\mathrm{asym}}=\log^{\frac{1}{2}}(N)N^{-\frac{1}{2d+2}}$, as it is 
straightforward to derive that the uniform grid admits $\rho(N)\in\mathcal{O}(N^{-\frac{1}{d+1}})$ 
and $\frac{1}{|\mathbb{B}_{\rho(N)}(\bm{x})|}\in\mathcal{O}(N^{-\frac{1}{d+1}})$. Since the asymptotic
posterior variance bound for the squared exponential kernel is rather loose as outlined in \cref{subsec:VarUniform}, $\ubar{v}_{\mathrm{asym}}$ is also conservative and we can observe 
a faster decay rate of the ultimate bound in the simulations. Nevertheless, Corollary~\ref{th:asstab}
is an important result since it guarantees a vanishing tracking error in the limit of infinite 
training data.

\subsection{Safety Region Evaluation in Robotic Manipulator Simulations}
\label{subsec:robot}

We investigate the tracking error bound in the real-world application of controlling
a robotic manipulator with 2 degrees of freedom (DoFs) in simulation. The considered 
robotic manipulator, whose dynamics are derived according to \cite[Chapter 4]{Murray1994}, 
has links with unit length and unit masses/ inertia. We assume that
the Lipschitz constant~$L_f$ is known and straightforwardly extend \cref{th:errbound_with} 
to the multidimensional case based on the union bound. Since the considered robot has
2 DoFs, its state space is four dimensional~$[q_1\ \dot{q}_1\ q_2\ \dot{q}_2]$ and 
we evaluate the uniform error bound on $\mathbb{X}=[-\pi\ \pi]^4$. The GP is 
trained using $81$ samples, which are equally spaced in the region~$[-1.5\ 1.5]^4$. 
For both DoFs, we use control gains $k_c=15$ and $\lambda=3$. The reference trajectories
for tracking are sinusoidal as shown in \cref{fig:RobotLyap} on the right side.

In order to visualize the tracking error bound in \cref{thm:Stable} in an illustrative 
way, we transform it from the state space into the task space as shown in 
\cref{fig:RobotLyap} on the left. Exploiting only the learned dynamics, it can be
guaranteed that the robot manipulator will not leave the 
depicted area, which thereby can be considered as safe. While similar theoretical
results can be derived using previous error bounds for GPs, there applicability 
to these practical settings is severely limited due to
\begin{enumerate}
\item they do not allow Gaussian observation noise on the training data~\citep{Srinivas2012}, 
which is commonly assumed in control,
\item they require constants, such as the maximal information gain in ~\citep{Srinivas2010},
 which cannot be computed efficiently in practice,
\item they base on assumptions, which are unintuitive and difficult to verify in practice (e.g., 
the RKHS norm of the unknown dynamical system~\citep{Berkenkamp2016a}).
\end{enumerate}

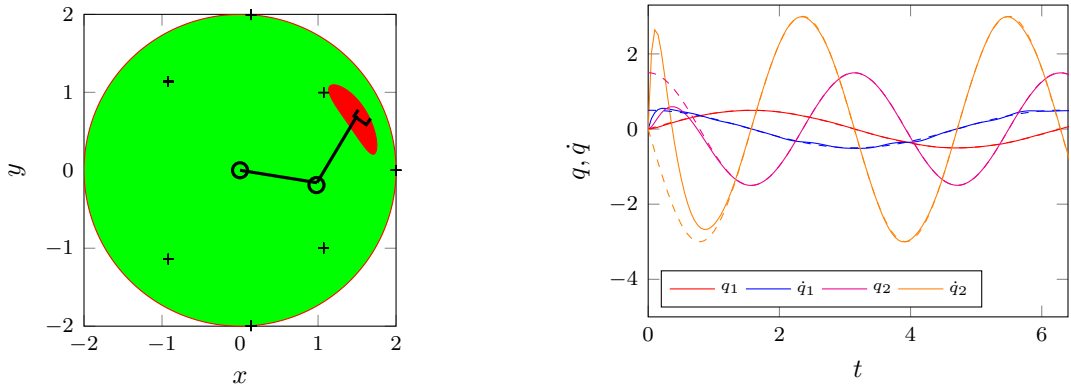
\begin{figure}
	\pgfplotsset{width=8\columnwidth/10,
		height = 80\columnwidth/100 }
	\begin{minipage}{0.47\textwidth}
		\centering 
		\def\file{plots/2DofRobot.txt}	
		\tikzsetnextfilename{RobotLyap}
		\def\msize{0.1cm}  	 \def\mx{1.54} \def\my{0.64}
		\begin{tikzpicture}
		\begin{axis}[grid=none, axis equal,
		xlabel={$x$}, ylabel={$y$}, legend pos = north west,
		xmin=-2, xmax = 2, ymin = -2, ymax = 2,	
		]
		\addplot[red,fill=green, forget plot] table[x=Xhullte1_1,y=Xhullte1_2]{\file};
		\addplot[red,fill=red, forget plot] table[x=Xhull1_1,y=Xhull1_2]{\file};
		\addplot[only marks,black, mark=+]table[x = Xtasktr_1,y=Xtasktr_2]{\file};
		\addplot[very thick, black] table[x=Xjoint11_1, y=Xjoint11_2]{\file};
		\addplot[very thick, black] table[x=Xjoint21_1 , y=Xjoint21_2 ]{\file};
		\draw[very thick] (axis cs:0,0) circle (0.1cm);
		\draw[very thick] (axis cs:0.98,-0.19) circle (0.1cm);
		\draw[very thick,rotate around={-35:(axis cs:\mx, \my)}] 
		(axis cs:\mx, \my) --++ (\msize,0cm)--++ (0cm,\msize);	
		\draw[very thick,rotate around={-35:(axis cs:\mx, \my)}] 
		(axis cs:\mx, \my) --++ (-\msize,0cm)--++ (0cm,\msize);				
		\end{axis}
		\end{tikzpicture}
	\end{minipage}\hfill
	\begin{minipage}{0.47\textwidth}
	\pgfplotsset{width=10\columnwidth/10,
		height = 80\columnwidth/100 }
		\centering 
		\def\file{plots/2DofRobot.txt}	
		\tikzsetnextfilename{RobotJoints}
		\begin{tikzpicture}
		\begin{axis}[grid=none,	xlabel={$t$}, ylabel={$q,\dot{q}$}, 
		legend pos = south west,xmin=0,xmax=6.4,ymin=-5.0,ymax=3.3,	
		legend columns=4]
		\addplot[red,] table[x=T,y=Xsim_1]{\file};
		\addplot[blue]table[x=T,y=Xsim_2]{\file};
		\addplot[magenta,] table[x=T,y=Xsim_3]{\file};
		\addplot[orange,] table[x=T,y=Xsim_4]{\file};
		\addplot[red,dashed,forget plot] table[x=T,y=Xd_1]{\file};
		\addplot[blue,dashed,forget plot] table[x=T,y=Xd_2]{\file};
		\addplot[magenta,dashed,forget plot] table[x=T,y=Xd_3]{\file};
		\addplot[orange,dashed,forget plot] table[x=T,y=Xd_4]{\file};
		\legend{$q_1$,$\dot{q}_1$,$q_2$, $\dot{q}_2$}
		\end{axis}
		\end{tikzpicture}
	\end{minipage}\
	
	\caption{The robot manipulator (left) is guaranteed to remain in~$\mathbb{V}$ (red) after a transient 
	phase, while the remaining task space~$\mathbb{X}\setminus\mathbb{V}$~(green) can be 
	considered as safe. The trajectories of the joint angles and velocities (right) approach the 
	reference trajectories (dashed lines) over time. 
		}
	\label{fig:RobotLyap}
\end{figure}

\section{Conclusion}
\label{sec:conclusion}
This paper presents a novel uniform error bound for Gaussian process regression. 
By exploiting the inherent probability distribution of Gaussian processes instead
of the RKHS attached to the covariance function, a wider class of functions can be 
considered. A novel method to compute interpretable hyperparameters based on 
prior knowledge is presented and the importance of suitably chosen hyperparameters is 
demonstrated for uniform error bounds. By deriving an analytical bound for the posterior 
variance of Gaussian processes and analyzing the asymptotic behavior of the bound it is 
shown that the derived uniform error bound converges to zero under weak assumptions on 
the training data distribution. The derived results are employed to develop a 
provably safe tracking control algorithm for which asymptotic stability in the 
limit of infinite training data is shown. The theoretical results are validated 
in simulations illustrating the behavior of the derived posterior variance bound, 
investigating the effect of badly chosen hyperparameters on uniform error bounds 
and demonstrating the safe tracking control of a robotic manipulator.

\appendix

\section{Proofs for Posterior Variance Bounds}

\subsection{Variance Bounds}

\begin{proof}[Proof of \cref{th:var_bound}]
	Since $\bm{K}_N+\sigma_n^2\bm{I}_N$ is a 
	positive definite, quadratic matrix, it follows 
	that
	\begin{align*}
	\sigma_{N}^2(\bm{x})&\leq k(\bm{x},\bm{x})-
	\frac{\left\|\bm{k}_N(\bm{x})\right\|^2}
	{\lambda_{\max}\left(\bm{K}_N\right)+\sigma_n^2}.
	\end{align*}
	Applying the Gershgorin theorem \citep{Gershgorin1931} 
	the maximal eigenvalue is bounded by
	\begin{align*}
	\lambda_{\max}(\bm{K}_N)\leq N
	\max\limits_{\bm{x}',\bm{x}''\in\mathbb{D}_N^x}
	k(\bm{x}',\bm{x}'').
	\end{align*}
	Furthermore, due to the definition of 
	$\bm{k}_N(\bm{x})$ we have
	\begin{align*}
	\|\bm{k}_N(\bm{x})\|^2\geq N 
	\min\limits_{\bm{x}'\in\mathbb{D}_N^x}
	k^2(\bm{x}',\bm{x}).
	\end{align*}
	Therefore, $\sigma_N^2(\bm{x})$ can be bounded by
	\begin{align}
	\sigma_{N}^2(\bm{x})&\leq k(\bm{x},\bm{x})-
	\frac{N\min\limits_{\bm{x}'\in\mathbb{D}_N^x}
		k^2(\bm{x}',\bm{x})}{N
		\max\limits_{\bm{x}',\bm{x}''\in\mathbb{D}_N^x}
		k(\bm{x}',\bm{x}'')+\sigma_n^2}.
	\label{eq:sigma_bound1}
	\end{align}
	This bound can be further simplified exploiting 
	the fact that $\sigma_N^2(\bm{x})\leq
	\sigma_{N-1}^2(\bm{x})$ \citep{Vivarelli1998} 
	and considering only samples inside the ball 
	$\mathbb{B}_{\rho}(\bm{x}^*)$ with radius 
	$\rho\in\mathbb{R}_+$. Using this reduced data 
	set instead of $\mathbb{D}_N^x$ and writing the 
	right side of \eqref{eq:sigma_bound1} as a single 
	fraction results in
	\begin{align}
	\label{eq:sigma_bound3}
	\sigma_{N}^2(\bm{x})&\leq\frac{k(\bm{x},\bm{x})
		\sigma_n^2+\left|\mathbb{B}_{\rho}(\bm{x}^*)\right|
		\xi(\bm{x},\rho)}{\left|\mathbb{B}_{\rho}
		(\bm{x}^*)\right|\max\limits_{\bm{x}',\bm{x}''
			\in\mathbb{B}_{\rho}(\bm{x})}k(\bm{x}',\bm{x}'')
		+\sigma_n^2},
	\end{align}
	where
	\begin{align*}
	\xi(\bm{x},\rho&)=
	k(\bm{x},\bm{x})\max\limits_{\bm{x}',\bm{x}''
		\in\mathbb{B}_{\rho}(\bm{x}^*)}k(\bm{x}',\bm{x}'')
	\!-\!\min\limits_{\bm{x}'\in \mathbb{B}_{\rho}
		(\bm{x}^*)}k^2(\bm{x}',\bm{x}).
	\end{align*}
	Under the assumption that $\rho\leq 
	\frac{k(\bm{x},\bm{x})}{L_k}$ it follows from 
	the Lipschitz continuity of $k(\cdot,\cdot)$ that
	\begin{align*}
	\min\limits_{\bm{x}'\in \mathbb{B}_{\rho}(\bm{x}^*)}
	k^2(\bm{x}',\bm{x})\geq (k(\bm{x}^*,\bm{x})-
	L_k\rho)^2.
	\end{align*}
	Furthermore, it holds that
	\begin{align*}
	\max\limits_{\bm{x}',\bm{x}''\in\mathbb{B}_{\rho}
		(\bm{x})}k(\bm{x}',\bm{x}'')\leq 
	k(\bm{x}^*,\bm{x}^*)+2L_k\rho.
	\end{align*}
	Therefore, $\xi(\bm{x},\rho)$ can be bounded by
	\begin{align*}
	\xi(\bm{x},\rho)&\leq 2L_k\rho k(\bm{x}^*,\bm{x})+2L_k\rho k(\bm{x},\bm{x})-L_k^2\rho^2
	\end{align*}
	since $k(\bm{x},\bm{x})k(\bm{x}^*,\bm{x}^*)=k^2(\bm{x}^*,\bm{x})$.
	Hence, the result is proven.
\end{proof}

\begin{proof}[Proof of Corollary~\ref{cor:var_bound}]
	The proof follows directly from 
	\eqref{eq:sigma_bound3} and the fact that 
	\begin{align*}
	\min\limits_{\bm{x}'\in\mathbb{B}_{\rho}(\bm{x})}
	k(\bm{x}',\bm{x})&\leq k(\rho)\\
	\max\limits_{\bm{x}',\bm{x}''\in\mathbb{B}_{\rho}
		(\bm{x})}k(\bm{x}',\bm{x}'')&= k(0)
	\end{align*}
	since the kernel is isotropic and decreasing.
\end{proof}

\subsection{Asymptotic Behavior of Variance Bounds}

\begin{proof}[Proof of \cref{th:varvan}]
	By setting $\bm{x}^*=\bm{x}$ in Theorem~\ref{th:var_bound} and simplifying 
	we obtain
	\begin{align*}
		\sigma_N^2(\bm{x})\leq 4L_k \rho +\frac{\sigma_n^2}{|\mathbb{B}_{\rho}(\bm{x})|}.
	\end{align*}
	Considering only the asymptotic behavior described by the $\mathcal{O}$-notation
	we therefore have
	\begin{align*}
		\sigma_N^2(\bm{x})=\mathcal{O}\left( \frac{1}{\alpha(N)} \right),
	\end{align*}
	which proofs the theorem.
\end{proof}

In order to prove Lemma~\ref{th:ball} , some auxiliary results
for binomial distributions are necessary. These
are provided in the following Lemmas.
\begin{lemma}
	\label{lem:Bernoulli}
	The $k$-th central moment of a Bernoulli distributed random variable $X$ is given by
	\begin{align}
	E[(X-E[X])^k]=\sum\limits_{i=0}^{k-1}(-1)^i\binom{k}{i}p^{i+1}+p^k
	\end{align}
\end{lemma}
\begin{proof}
	The polynom $(X-E[X])^k$ can be expanded as
	\begin{align*}
	(X-&E[X])^k=\sum\limits_{i=0}^k \binom{k}{i}(-1)^{i}X^{k-i}E[X]^i.
	\end{align*}
	The $k$-th moment about the origin of the Bernoulli distribution is given by $p$ for $k>0$ \citep{Forbes2011}. Therefore, the expectation of this polynomial is given by
	\begin{align*}
	E[(X-&E[X])^k]=\sum\limits_{i=0}^{k-1} \binom{k}{i}(-1)^{i}pp^i+p^k,
	\end{align*}
	which directly yields the result.
\end{proof}

\begin{lemma}
	\label{lem2}
	The $2k$-th central moment of a binomial distributed random variable $M$ with $N>2k$ samples is bounded by
	\begin{align}
	E[(X-E[X])^{2k}]\leq \sum\limits_{m=1}^k(Np)^m\alpha_m
	\label{eq3}
	\end{align}
	where $\alpha_m\in\mathbb{R}$ are finite coefficients.
\end{lemma}
\begin{proof}
	A binomial random variable is defined as the sum of $N$ i.i.d. Bernoulli random variables $X_i$. Therefore, the $2k$-th central moment of the binomial distribution is given by
	\begin{align*}
	E[(M-E[M])^{2k}]=E\left[\left(\sum\limits_{i=1}^{N}(X_i-p)\right)^{2k}\right].
	\end{align*}
	Define the multinomial coefficient as
	\begin{align*}
	\binom{N}{i_1,\ldots,i_k}=\frac{N!}{\prod\limits_{j=1}^ki_j!}.
	\end{align*}
	Then, the sum in the expectation can be expanded, which yields
	\begin{align}
	E[(M-E[M])^{2k}]=
	\sum\limits_{i_1+\ldots+i_{N}=2k}\!\binom{2k}{i_1,\ldots,i_{N}}\prod\limits_{j=1}^{N}E\!\left[\left(X_j-p)\right)^{i_j}\right]\!.
	\label{eq2}
	\end{align}
	This equation expresses the moments of the binomial distribution in terms of the moments of the Bernoulli distribution. Since the first central moment of every distribution equals $0$, summands containing a $i_j=1$ equal $0$. Therefore, we obtain the equality
	\begin{align}
	E[(M-E[M])^{2k}]=
	\sum\limits_{\subalign{i_1+\ldots+i_{N}=2k\\i_j\neq 1\forall j=1,\ldots,N}}\!\binom{2k}{i_1,\ldots,i_{N}}\!\prod\limits_{i_j>1}\!E\!\left[\left(X_j-p)\right)^{i_j}\right]\!.
	\label{eq1}
	\end{align}
	Moreover, we have
	\begin{align*}
	E[(X-E[X])^k]=p h_k(p)
	\end{align*}
	with
	\begin{align*}
	h_k(p)=\sum\limits_{i=0}^{k-1}(-1)^i\binom{k}{i}p^{i}+p^{k-1}
	\end{align*}
	due to Lemma~\ref{lem:Bernoulli}. By substituting this into \eqref{eq1} we obtain  
	\begin{align*}
	E[(M-E[M])^{2k}]=
	\sum\limits_{\subalign{i_1+\ldots+i_{N}=2k\\i_j\neq 1\forall j=1,\ldots,N}}\binom{2k}{i_1,\ldots,i_N}\prod\limits_{i_j>1}ph_{i_j}(p).
	\end{align*}
	The product can have between $1$ and $k$ factors due to the structure of the problem. Therefore, it is not necessary for the sum to consider all $N$ coefficients $i_j$, but rather consider only $1\leq m\leq k$ coefficients which are greater than $1$. This leads to the following equality
	\begin{align}
	E[(M-E[M])^{2k}]=
	\sum\limits_{m=1}^k\!\binom{N}{m}p^m\!\sum\limits_{\subalign{i_1+\ldots+i_{m}=2k\\i_j> 1\forall j=1,\ldots,m}}\!\binom{2k}{i_1,\ldots,i_m}\!\prod\limits_{i_j>1}\!h_{i_j}\!(p).
	\label{eq5}
	\end{align}
	Due to \citep{Cormen2009} it holds that $\binom{N}{m}\leq \frac{N^m}{m!}$. Furthermore, the functions $h_k(\cdot)$ can be upper bounded by $\sum\limits_{i=1}^k\binom{k}{i}=2^k$ because $0\leq p\leq 1$. Therefore, we can upper bound the $2k$-th central moment of the binomial distribution by	
	\begin{align*}
	E[(M-E[M])^{2k}]\leq \sum\limits_{m=1}^k(Np)^m\alpha_m
	\end{align*}
	with
	\begin{align}
	\label{eq4}
	\alpha_m&=\frac{\sum\limits_{\subalign{i_1+i_2+\ldots+i_{m}=2k\\i_j> 1\forall j=1,\ldots,m}}\binom{2k}{i_1,i_2,\ldots,i_m}\prod\limits_{i_j>1}2^{i_j}}{m!}
	\end{align}
	and the result is proven.
\end{proof}
The restriction to $N>2k$ samples allows to derive a relatively simple expression for the expansion in \eqref{eq2}. However, the bound \eqref{eq3} also holds without this condition, since it only guarantees that for $i_j=1$, $\forall j=1,\ldots,N$, $\sum\limits_{j=1}^N i_j\geq 2k$ and therefore, all possible combinations of $i_j$ can be estimated simpler in \eqref{eq5}. Hence, the corresponding summands in \eqref{eq4} can be considered $0$ for $N\leq 2k$ and the upper bound \eqref{eq3} still holds for $N\leq 2k$.

Based on these preliminary results we can state the proof of Lemma~\ref{th:ball}.

\begin{proof}[Proof of Lemma~\ref{th:ball}]
	We have to show 
	that the number of samples from the probability 
	distribution with density $p(\cdot)$ inside the 
	balls with radius $\rho(N)$ grows to infinity 
	sufficiently fast. The number of samples 
	$|\mathbb{B}_{\rho(N)}(\bm{x})|$ follows a 
	binomial distribution with mean
	\begin{align*}
	E\left[\left|\mathbb{B}_{\rho(N)}(\bm{x})
	\right|\right]&=N\tilde{p}(N),
	\end{align*}
	where
	\begin{align*}
	\tilde{p}(N)=\int\limits_{\{\bm{x}'\in\mathbb{X}:
		\|\bm{x}-\bm{x}'\|\leq \rho(N)\}}p(\bm{x}')
	\mathrm{d}\bm{x}'
	\end{align*}
	is the probability of a sample lying inside 
	the ball around $\bm{x}$ with radius $\rho(N)$ 
	for fixed $N\in\mathbb{N}$. Since we have 
	\begin{align*}
	\int\limits_{\{\bm{x}'\in\mathbb{X}:\|\bm{x}
		-\bm{x}'\|\leq \rho(N)\}}p(\bm{x}')\mathrm{d}\bm{x}'
	&\geq cN^{-1+\epsilon}
	\end{align*}
	by assumption, this mean exhibits a rate
	\begin{align*}
	E\left[\left|\mathbb{B}_{\rho(N)}(\bm{x})\right|\right]
	=\mathcal{O}\left(N^{\epsilon}\right).
	\end{align*}
	Therefore, it is sufficient to show that 
	$|\mathbb{B}_{\rho(N)}(\bm{x})|$ converges to 
	its expectation almost surely, which is 
	identically to proving that 
	\begin{align*}
	\lim\limits_{N\rightarrow\infty}\frac{
		|\mathbb{B}_{\rho(N)}(\bm{x})|}{E\left[
		\left|\mathbb{B}_{\rho(N)}(\bm{x})\right|
		\right]}=1\quad a.s.
	\end{align*}
	Due to the Borel-Cantelli lemma, this convergence 
	is guaranteed if 
	\begin{align}
	\sum\limits_{N=1}^{\infty}P\left( \left|\frac{
		|\mathbb{B}_{\rho(N)}(\bm{x})|}{E\left[\left|
		\mathbb{B}_{\rho(N)}(\bm{x})\right|\right]}-1\right|
	>\xi \right)<\infty
	\label{eq:bclemma}
	\end{align}
	holds for all $\xi>0$. The probability for each 
	$N\in\mathbb{N}$ can be bounded by 
	\begin{align*}
	P\Bigg( \Bigg| \frac{|\mathbb{B}_{\rho(N)}(\bm{x})|}
	{N\tilde{p}(N)}\!- \!1\Bigg|&\!>\!\xi \Bigg)\leq 
	\frac{	E\left[\left(\left|\mathbb{B}_{\rho(N)}
		(\bm{x})\right|\!-\!N\tilde{p}(N)\right)^{2k}\right]}
	{(\xi N\tilde{p}(N))^{2k}}.
	\end{align*}
	for each $k\in\mathbb{N}_+$ due to Chebyshev's
	inequality, where the $2k$-th central moment of 
	the binomial distribution can be bounded by
	\begin{align*}
	E\left[\left(\left|\mathbb{B}_{\rho(N)}(\bm{x})
	\right|-N\tilde{p}(N)\right)^{2k}\right]\leq
	\sum\limits_{i=1}^{k}\alpha_i\tilde{p}^i(N)
	N^i
	\end{align*} 
	with some coefficients $\alpha_i<\infty$ due to
	Lemma~\ref{lem2}.
	Therefore, we can bound each probability in 
	\eqref{eq:bclemma} by
	\begin{align*}
	P\Bigg( \Bigg| \frac{|\mathbb{B}_{\rho(N)}
		(\bm{x})|}{N\tilde{p}(N)}-1\Bigg|&>\xi \Bigg)
	\leq \sum\limits_{i=1}^{k}\alpha_i
	\tilde{p}^{-2k+i}(N)N^{-2k+i}.
	\end{align*}
	Due to \eqref{eq:cond2newmain} this bound can be 
	simplified to 
	\begin{align*}
	P\Bigg( \Bigg| \frac{|\mathbb{B}_{\rho(N)}
		(\bm{x})|}{N\tilde{p}(N)}-1\Bigg|&>\xi \Bigg)
	\leq N^{-k\epsilon}\sum\limits_{i=0}^{k-1}
	\tilde{\alpha}_{k-i}N^{-i\epsilon},
	\end{align*}
	where $\tilde{\alpha}_i=c^{-2k+i}\alpha_i$.
	Let $k=\left\lceil\frac{1}{\epsilon}\right\rceil+1$. 
	Then, each exponent is smaller than or equal to 
	$-1-\epsilon$. Hence, the sum of probabilities can 
	be bounded by 
	\begin{align*}
	\sum\limits_{N=1}^{\infty}P\Bigg( \Bigg| 
	\frac{|\mathbb{B}_{\rho(N)}(\bm{x})|}
	{N\tilde{p}(N)}-1\Bigg|&>\xi \Bigg)\leq 
	\sum\limits_{i=0}^{k-1}\tilde{\alpha}_{k-i} 
	\zeta\big((k+i)\epsilon\big),
	\end{align*}
	where $\zeta(\cdot)$ is the Riemann zeta function, 
	which has finite values. Therefore, we obtain
	\begin{align*}
	\sum\limits_{N=1}^{\infty}P\Bigg( \Bigg|
	\frac{|\mathbb{B}_{\rho(N)}(\bm{x})|}
	{N\tilde{p}(N)}-1\Bigg|>\epsilon \Bigg)< 
	\infty
	\end{align*}
	and consequently, the theorem is proven.
\end{proof}

\begin{proof}[Proof of Corollary~\ref{cor:varvan}]
	Let 
	\begin{align*}
	\bar{p}&=\min\limits_{\|\bm{x}-\bm{x}'\|\leq 
		\xi}p(\bm{x}')\\
	\tilde{p}(N)&=\int\limits_{\{\bm{x}'\bm{x}'
		\in\mathbb{X}:\|\bm{x}-\bm{x}'\|\leq 
		\xi\}}p(\bm{x}')\mathrm{d}\bm{x}',
	\end{align*}
	where $\bar{p}$ is positive by assumption. Then, 
	we can bound $\tilde{p}(N)$ by
	\begin{align*}
	\tilde{p}(N)\geq \bar{p}V_{d}\rho^{d}(N),
	\end{align*}
	where $V_{d}$ is the volume of the $d$ 
	dimensional unit ball. Since $\rho(N)\geq c
	N^{-\frac{1}{d}+\epsilon}$ for some 
	$c,\epsilon>0$ by assumption, it follows that 
	\begin{align*}
	\tilde{p}(N)\geq \bar{p}V_{d}cN^{-1+
		\frac{\epsilon}{d}}.
	\end{align*}
	Hence, $\tilde{p}(N)$ satisfies the conditions 
	of Lemma~\ref{th:ball}, which proves the corollary.
\end{proof}

\section{Proofs for the Probabilistic Uniform Error Bound}

\subsection{Hyperparameter Bounds}

In order to make the proof of Theorem~\ref{th:fmax} easier accessible, we 
split it up into the following auxiliary lemma and the main proof. The 
auxiliary lemma derives a bound for the expected supremum of a Gaussian process.
\begin{lemma}
	\label{lem:expsup}
	Consider a Gaussian process with a continuously differentiable covariance 
	function $k(\cdot,\cdot)$ and let $L_k$ denote its Lipschitz constant on 
	the set $\mathbb{X}$ with maximum extension  
	\mbox{$\theta=\max_{\bm{x},\bm{x}'\in\mathbb{X}}\|\bm{x}-\bm{x}'\|$}. 
	Then, the expected supremum of a sample function $f(\cdot)$ of this
	Gaussian process satisfies
	\begin{align*}
	E\left[\sup\limits_{\bm{x}\in \mathbb{X}}f(\bm{x})\right]\leq 
	12\sqrt{d}	\max\limits_{\bm{x}\in\mathbb{X}}\sqrt{k(\bm{x},\bm{x})}\sqrt{2\log\left( \frac{\sqrt{4\theta L_k}(1+\sqrt{2})\mathrm{e}}{\max\limits_{\bm{x}\in\mathbb{X}}\sqrt{k(\bm{x},\bm{x})}} \right) }.
	\end{align*}
\end{lemma}
\begin{proof}
	We prove this lemma by making use of the metric entropy criterion for the 
	sample continuity of some version of a Gaussian process \citep{Dudley1967}.
	This criterion allows to bound the expected supremum of a sample function $f(\cdot)$
	by
	\begin{equation}
	\mathrm{E}\left[ \sup\limits_{\bm{x}\in\mathbb{X}}f(\bm{x}) \right]\leq \int\limits_0^{\max\limits_{\bm{x}\in\mathbb{X}}\sqrt{k(\bm{x},\bm{x})}}
	\sqrt{\log(N(\varrho,\mathbb{X}))}\mathrm{d}\varrho,
	\label{eq:metEntropy}
	\end{equation}
	where $N(\varrho,\mathbb{X})$ is the $\varrho$-packing number of $\mathbb{X}$ 
	with respect to the covariance pseudo-metric
	\begin{align*}
	d_k(\bm{x},\bm{x}')=\sqrt{k(\bm{x},\bm{x})+k(\bm{x}',\bm{x}')-2k(\bm{x},\bm{x}')}.
	\end{align*}
	Instead of bounding the $\varrho$-packing number, we bound the $\varrho/2$-covering 
	number, which is known to be an upper bound. The covering number can be easily bounded 
	by transforming the problem of covering $\mathbb{X}$ with respect to the pseudo-metric 
	$d_k(\cdot,\cdot)$ into a coverage problem in the original metric of $\mathbb{X}$. For 
	this reason, define
	\begin{align*}
	\psi(\varrho')=\sup\limits_{\subalign{\bm{x},\bm{x}'&\in \mathbb{X}\\ 
			\|\bm{x}-\bm{x}'&\|_{\infty}\leq \varrho'}} d_k(\bm{x},\bm{x}'),
	\end{align*}
	which is continuous due to the continuity of the covariance kernel $k(\cdot,\cdot)$. 
	Consider the inverse function
	\begin{align*}
	\psi^{-1}(\varrho)=\inf\left\{\varrho'>0:~\psi(\varrho')>\varrho\right\}.
	\end{align*}
	Continuity of $\psi(\cdot)$ implies $\varrho=\psi(\psi^{-1}(\varrho))$. In particular, 
	this means that we can guarantee $d_k(\bm{x},\bm{x}')\leq \frac{\varrho}{2}$ if 
	\mbox{$\|\bm{x}-\bm{x}'\|\leq \psi^{-1}(\frac{\varrho}{2})$}. Due to this relationship 
	it is sufficient to construct an uniform grid with grid constant $2\psi^{-1}(\frac{\varrho}{2})$ 
	in order to obtain a $\varrho/2$-covering net of $\mathbb{X}$. Furthermore, the cardinality 
	of this grid is an upper bound for the $\varrho/2$-covering number, i.e.
	\begin{align*}
	M(\varrho/2,\mathbb{X})\leq 
	\left\lceil \frac{\theta}{2\psi^{-1}(\frac{\varrho}{2})} \right\rceil^{d}.
	\end{align*}
	Therefore, it follows that
	\begin{align*}
	N(\varrho,\mathbb{X})\leq 
	\left\lceil \frac{\theta}{2\psi^{-1}(\frac{\varrho}{2})} \right\rceil^{d}.
	\end{align*}
	Due to the Lipschitz continuity of the covariance function, we can bound 
	$\psi(\cdot)$ by
	\begin{align*}
	\psi(\varrho')&\leq \sqrt{2L_k\varrho'}.
	\end{align*}
	Hence, the inverse function satisfies
	\begin{align*}
	\psi^{-1}\left(\frac{\varrho}{2}\right)\geq 
	\left(\frac{\varrho}{2\sqrt{2L_k}}\right)^2
	\end{align*}
	and consequently
	\begin{align*}
	N(\varrho,\mathbb{X})\leq \left(1+\frac{4\theta L_k}{\varrho^2}\right)^{d}
	\end{align*}
	holds, where the ceil operator is resolved through the addition of $1$.
	Substituting this expression in the metric entropy bound 
	\eqref{eq:metEntropy} yields
	\begin{align*}
	E\left[\sup\limits_{\bm{x}\in \mathbb{X}}f(\bm{x})\right]\leq 12\sqrt{d}
	\int\limits_0^{\max\limits_{\bm{x}\in\mathbb{X}}\sqrt{k(\bm{x},\bm{x})}}
	\sqrt{\log\left(1+\frac{4\theta L_k}{\varrho^2}\right)}\mathrm{d}\varrho.
	\end{align*}
	This integral can be solved similarly as shown in \citep{Grunewalder2010} 
	using the inequality
	\begin{align*}
		\int\limits_0^{\max\limits_{\bm{x}\in\mathbb{X}}\sqrt{k(\bm{x},\bm{x})}}
		\sqrt{\log\left(1+\frac{4\theta L_k}{\varrho^2}\right)}\mathrm{d}\varrho\leq 
		\int\limits_0^{\max\limits_{\bm{x}\in\mathbb{X}}\sqrt{k(\bm{x},\bm{x})}}
		\sqrt{\log\left(\left(1+\sqrt{2}\right)^{2}\frac{4\theta L_k}{\varrho^2}\right)}\mathrm{d}\varrho.
	\end{align*}
	Through a change of variables we obtain
	\begin{align*}
		\int\limits_0^{\max\limits_{\bm{x}\in\mathbb{X}}\sqrt{k(\bm{x},\bm{x})}}
		\sqrt{\log\left(\left(1+\sqrt{2}\right)^{2}\frac{4\theta L_k}{\varrho^2}\right)}\mathrm{d}\varrho\leq
		(1+\sqrt{2})\sqrt{4\theta L_k} \int\limits_{0}^{\frac{\max\limits_{\bm{x}\in\mathbb{X}}\sqrt{k(\bm{x},\bm{x})}}{(1+\sqrt{2})\sqrt{4\theta L_k}}} \sqrt{-\log(\tilde{\varrho}^{-2})}\mathrm{d}\tilde{\varrho}.
	\end{align*}
	This integral can be bounded
	\begin{align*}
		(1+\sqrt{2})\sqrt{4\theta L_k}\int\limits_{0}^{\frac{\max\limits_{\bm{x}\in\mathbb{X}}\sqrt{k(\bm{x},\bm{x})}}{(1+\sqrt{2})\sqrt{4\theta L_k}}} \sqrt{-\log(\tilde{\varrho}^{-2})}\mathrm{d}\tilde{\varrho}\leq
		\max\limits_{\bm{x}\in\mathbb{X}}\sqrt{k(\bm{x},\bm{x})}\sqrt{2\log\left( \frac{\sqrt{4\theta L_k}(1+\sqrt{2})\mathrm{e}}{\max\limits_{\bm{x}\in\mathbb{X}}\sqrt{k(\bm{x},\bm{x})}} \right) }
	\end{align*}
	which concludes the proof.
\end{proof}
Based on the expected supremum of Gaussian process it is possible to 
derive a high probability bound for the supremum of a sample function
which we exploit to proof Theorem~\ref{th:fmax} in the following.
\begin{proof}[Proof of Theorem~\ref{th:fmax}]
	We prove this lemma by exploiting the wide theory of concentration inequalities 
	to derive a bound for the supremum of the sample function $f(\bm{x})$. We apply the 
	Borell-TIS inequality \citep{Talagrand1994}
	\begin{align}
	P\Bigg( \sup\limits_{\bm{x}\in\mathbb{X}}f(\bm{x})-
	E\Bigg[ \sup\limits_{\bm{x}\in\mathbb{X}}f&(\bm{x}) \Bigg] \geq 
	c \Bigg)\leq \exp\left( -\frac{c^2}{2\max\limits_{\bm{x}\in\mathbb{X}}
		k(\bm{x},\bm{x})} \right).
	\label{eq:Tal}
	\end{align}
	Due to Lemma~\ref{lem:expsup} we have
	\begin{align}
	\label{eq:Esup}
	E\left[\sup\limits_{\bm{x}\in \mathbb{X}}f(\bm{x})\right]\leq 
	12\sqrt{d}	\max\limits_{\bm{x}\in\mathbb{X}}\sqrt{k(\bm{x},\bm{x})}\sqrt{2\log\left( \frac{\sqrt{4\theta L_k}(1+\sqrt{2})\mathrm{e}}{\max\limits_{\bm{x}\in\mathbb{X}}\sqrt{k(\bm{x},\bm{x})}} \right) }.
	\end{align}
	We conclude the proof of this lemma by
	substituting \eqref{eq:Esup} in \eqref{eq:Tal} and choosing 
	$c=\sqrt{2\log\left( \frac{1}{\delta_f} \right)}\max\limits_{\bm{x}\in\mathbb{X}}\sqrt{k(\bm{x},\bm{x})}$.
\end{proof}
For the proof of Corollary~\ref{th:Lip_f} we make use of the fact that the derivative 
of a sample function can be considered as  a sample function from another Gaussian process.
Thus, we merely have to apply Theorem~\ref{th:fmax} to this new Gaussian process as 
shown in the following.
\begin{proof}[Proof of Corollary~\ref{th:Lip_f}]
	It has been shown in \cite[Theorem 5]{Ghosal2006} that the derivative functions 
	$\frac{\partial}{\partial x_i}f(\bm{x})$ are samples from derivative Gaussian 
	processes with covariance functions
	\begin{align*}
	k_{\partial i}(\bm{x},\bm{x}')=
	\frac{\partial^2}{\partial x_i\partial x_i'}k(\bm{x},\bm{x}').
	\end{align*}
	Therefore, we can apply \cref{th:fmax} to each of the derivative processes which 
	proves the corollary.
\end{proof}

\subsection{Error Bound}

\begin{proof}[Proof of Theorem~\ref{th:errbound_with}]
	We first prove the Lipschitz constant bounds of the posterior mean $\nu_N(\bm{x})$ and 
	the posterior variance $\sigma_N^2(\bm{x})$, before
	we derive the bound of the regression error. The norm of the difference between 
	the posterior mean $\nu_N(\bm{x})$ evaluated at two different points is given by
	\begin{align*}
	\|\nu_N(\bm{x})-\nu_N(\bm{x}')\|&=
	\left\| \left(\bm{k}(\bm{x},\bm{X}_N)-\bm{k}(\bm{x}',\bm{X}_N)\right)
	\bm{\alpha}\right\|
	\end{align*}
	with 
	\begin{align*}
	\bm{\alpha}=(\bm{K}(\bm{X}_N,\bm{X}_N)+\sigma_n^2\bm{I}_N)^{-1}\bm{y}_N.
	\end{align*}
	Due to the Cauchy-Schwarz inequality and the Lipschitz continuity of the 
	kernel we obtain
	\begin{align*}
	\|\nu_N(\bm{x})-\nu_N(\bm{x}')\|&\leq L_k\sqrt{N} \left\| \bm{\alpha} 
	\right\|\|\bm{x}-\bm{x}'\|,
	\end{align*}
	which proves Lipschitz continuity of the mean $\nu_N(\bm{x})$. In order to bound
	the Lipschitz constant of the posterior variance we employ the Cauchy-Schwarz 
	inequality to the absolute value of the difference of variances such that
	\begin{align}
	&|\sigma_N^2(\bm{x})-\sigma_N^2(\bm{x}')|\leq |k(\bm{x},\bm{x})-k(\bm{x}',\bm{x}')|\nonumber\\
	&+\!\left\|\bm{k}(\bm{x},\bm{X}_N)-\bm{k}(\bm{x}',\bm{X}_N)\right\|
	\left\|(\bm{K}(\bm{X}_N,\bm{X}_N)+\sigma_n^2\bm{I}_N)^{-1}\right\|
	\left\|\bm{k}(\bm{X}_N,\bm{x})+\bm{k}(\bm{X}_N,\bm{x}')\right\|.
	\label{eq:sigdiff}
	\end{align}
	On the one hand, we have
	\begin{align}
	\|\bm{k}(\bm{x},\bm{X}_N)-\bm{k}(\bm{x}',\bm{X}_N)\|\leq 
	\sqrt{N}L_k\|\bm{x}-\bm{x}'\|
	\label{eq:sigminus}
	\end{align}
	due to Lipschitz continuity of $k(\bm{x},\bm{x}')$. On the other hand 
	we have
	\begin{align}
	\|\bm{k}(\bm{x},\bm{X}_N)+\bm{k}(\bm{x}',\bm{X}_N)\|\leq 2\sqrt{N}
	\max\limits_{\bm{x},\bm{x}'\in\mathbb{X}}k(\bm{x},\bm{x}').
	\label{eq:sigplus}
	\end{align}
	The bound for the Lipschitz constant $L_{\sigma_N^2}$ follows from 
	substituting \eqref{eq:sigminus} and \eqref{eq:sigplus} in \eqref{eq:sigdiff}. 
	The Lipschitz continuity of the posterior variance can be transferred to the 
	posterior standard deviation. In order to see this more clearly 
	observe that the difference of the variance at two points 
	$\bm{x},\bm{x}'\in\mathbb{X}$ can be expressed as 
	\begin{align*}
	|\sigma_N^2(\bm{x})-\sigma_N^2(\bm{x}')|&=|\sigma_N(\bm{x})-\sigma_N(\bm{x}')||
	\sigma_N(\bm{x})+\sigma_N(\bm{x}')|.
	\end{align*}
	Since the standard deviation is positive semidefinite we have
	\begin{align*}
	|\sigma_N(\bm{x})+\sigma_N(\bm{x}')|\geq |\sigma_N(\bm{x})-\sigma_N(\bm{x}')|
	\end{align*}
	and hence, we obtain
	\begin{align*}
	|\sigma_N^2(\bm{x})-\sigma_N^2(\bm{x}')|\geq |\sigma_N(\bm{x})-\sigma_N(\bm{x}')|^2.
	\end{align*}
	Therefore, the difference of the posterior standard deviation at two points
	can be bounded by taking the square root of the Lipschitz constant multiplied
	with the distance between both points. Finally, we use these properties to prove
	the probabilistic uniform error bound by exploiting the fact that for every 
	grid $\mathbb{X}_{\tau}$ with $|\mathbb{X}_{\tau}|$ grid points and 
	\begin{align}
	\max\limits_{\bm{x}\in\mathbb{X}} \min\limits_{\bm{x}'\in\mathbb{X}_{\tau}}
	\|\bm{x}-\bm{x}'\|\leq \tau
	\label{eq:gridconstant}
	\end{align}
	it holds with probability of at least
	$1-|\mathbb{X}_{\tau}|\mathrm{e}^{-\beta(\tau)/2}$ that \citep{Srinivas2012}
	\begin{align*}
	|f(\bm{x})-\nu_{N}(\bm{x})|\leq \sqrt{\beta(\tau)}\sigma_{N}(\bm{x}) 
	\quad \forall\bm{x}\in \mathbb{X}_{\tau}.
	\end{align*}
	Choose \mbox{$\beta(\tau)=2\log\left(\frac{|\mathbb{X}_{\tau}|}{\delta}\right)$}, 
	then
	\begin{align*}
	|f(\bm{x})-\nu_{N}(\bm{x})|\leq \sqrt{\beta(\tau)}\sigma_{N}(\bm{x}) 
	\quad \forall\bm{x}\in \mathbb{X}_{\tau}
	\end{align*}
	holds with probability of at least $1-\delta$. Due to continuity of 
	$f(\bm{x})$, $\nu_N(\bm{x})$ and $\sigma_N(\bm{x})$ we obtain
	\begin{align*}
	\min\limits_{\bm{x}'\in\mathbb{X}_{\tau}}|f(\bm{x})-f(\bm{x}')|&\leq 
	\tau L_f\quad \forall \bm{x}\in\mathbb{X}\\
	\min\limits_{\bm{x}'\in\mathbb{X}_{\tau}}|\nu_N(\bm{x})-\nu_N(\bm{x}')|&\leq 
	\tau L_{\nu_N}\quad \forall \bm{x}\in\mathbb{X}\\
	\min\limits_{\bm{x}'\in\mathbb{X}_{\tau}}|\sigma_N(\bm{x})-\sigma_N(\bm{x}')|&\leq 
	\sqrt{L_{\sigma_N^2}\tau}\quad \forall \bm{x}\in\mathbb{X}.
	\end{align*}
	Moreover, the minimum number of grid points satisfying \eqref{eq:gridconstant} is 
	given by the covering number $M(\tau,\mathbb{X})$. Hence, we obtain
	\begin{align*}
	P\left(|g(\bm{x})-\nu_{N}(\bm{x})|\leq 
	\sqrt{\beta(\tau)}\sigma_{N}(\bm{x})+\gamma(\tau), 
	~\forall\bm{x}\in\mathbb{X}\right)\geq 1-\delta,
	\end{align*}
	where
	\begin{align*}
	\beta(\tau)&=2\log\left(\frac{M(\tau,\mathbb{X})}{\delta}\right)\\
	\gamma(\tau)&=(L_f+L_{\nu_N})\tau+\sqrt{\beta(\tau)L_{\sigma_N^2}\tau}.
	\end{align*}
\end{proof}

\subsection{Asymptotic Convergence}

\begin{proof}[Proof of Theorem~\ref{th:as_err}]
	Due to Theorem~\ref{th:errbound_with} with 
	$\beta_N(\tau)=\log\left( \frac{M(\tau,\mathbb{X})\pi^2N^2}{3\delta} \right)$
	and the union bound over all $N>0$ it follows that
	\begin{align}
	\sup\limits_{\bm{x}\in \mathbb{X}}|f(\bm{x})-\nu_{N}(\bm{x})|\leq
	\sqrt{\beta_N(\tau)}\sigma_{N}(\bm{x})+\gamma_N(\tau)\quad \forall N>0
	\label{eq:un_err}
	\end{align}
	with probability of at least $1-\delta/2$ for $\delta\in(0,1)$. A trivial 
	bound for the covering number can be obtained by considering a uniform grid over the 
	cube containing $\mathbb{X}$. This approach leads to
	\begin{align*}
	M(\tau,\mathbb{X})\leq \left(\frac{\theta\sqrt{d}}{2\tau}\right)^{d},
	\end{align*}
	where \mbox{$\theta=\max_{\bm{x},\bm{x}'\in\mathbb{X}}\|\bm{x}-\bm{x}'\|$}. Therefore,
	we have
	\begin{align}
	\beta_N(\tau)\leq d\log\left(\frac{\theta\sqrt{d}}{2\tau}\right)-\log(3\delta)+2\log(\pi N).
	\label{eq:beta(tau)}
	\end{align}
	Furthermore, the Lipschitz constant $L_{\nu_N}$ is bounded by 
	\begin{align*}
	L_{\nu_N}&\leq L_k\sqrt{N} 
	\left\| (\bm{K}(\bm{X}_N,\bm{X}_N)+\sigma_n^2\bm{I}_N)^{-1}\bm{y}_N \right\|
	\end{align*}
	due to Theorem~\ref{th:errbound_with}. Since the Gram matrix $\bm{K}(\bm{X}_N,\bm{X}_N)$ 
	is positive semidefinite and $f(\cdot)$ is bounded by some $\bar{f}$ with probability 
	of at least $1-\delta_f$ for $\delta_f\in(0,1)$ due to Theorem~\ref{th:fmax}, we can 
	bound $\left\| (\bm{K}(\bm{X}_N,\bm{X}_N)+\sigma_n^2\bm{I}_N)^{-1}\bm{y}_N \right\|$ by
	\begin{align*}
	\left\| (\bm{K}(\bm{X}_N,\bm{X}_N)+\sigma_n^2\bm{I}_N)^{-1}\bm{y}_N \right\|&
	\leq\frac{\|\bm{y}_N\|}
	{\rho_{\min}(\bm{K}(\bm{X}_N,\bm{X}_N)+\sigma_n^2\bm{I}_N)}\nonumber\\
	&\leq \frac{\sqrt{N}\bar{f}
		+\|\bm{\xi}_N\|}{\sigma_n^2},
	\end{align*}
	where $\bm{\xi}_N$ is a vector of $N$ i.i.d. zero mean Gaussian random variables with
	variance $\sigma_n^2$. Therefore, it follows that 
	$\frac{\|\bm{\xi}_N\|^2}{\sigma_n^2}\sim\chi_N^2$. Due to 
	\citep{Laurent2000}, with probability of at least $1-\exp(-\eta_N)$ we have 
	\begin{align*}
	\|\bm{\xi}_N\|^2\leq \left(2\sqrt{N\eta_N}+2\eta_N+N\right)\sigma_n^2.
	\end{align*}
	Setting $\eta_N=\log(\frac{\pi^2N^2}{3\delta})$ and applying the union bound 
	over all $N>0$ yields
	\begin{align*}
	\left\| (\bm{K}(\bm{X}_N,\bm{X}_N)+\sigma_n^2\bm{I}_N)^{-1}\bm{y}_N \right\|\leq
	\frac{\sqrt{N}\bar{f}+\sqrt{2\sqrt{N\eta_N}+2\eta_N+N}\sigma_n}
	{\sigma_n^2}\quad \forall N>0
	\end{align*}
	with probability of at least $1-\delta-\delta_f$. Hence, the Lipschitz constant of the 
	posterior mean function $\nu_N(\cdot)$ satisfies with probability of at least 
	$1-\delta-\delta_f$
	\begin{align*}
	L_{\nu_N}\leq L_k\frac{N\bar{f}+\sqrt{N(2\sqrt{N\eta_N}+2\eta_N+N)}\sigma_n}
	{\sigma_n^2}\quad \forall N>0.
	\end{align*}
	Since $\eta_N$ grows logarithmically with the number of training samples 
	$N$, it holds that $L_{\nu_N}\in\mathcal{O}(N)$ with 
	probability of at least $1-\delta-\delta_f$. The Lipschitz constant of the posterior
	variance $L_{\sigma_N^2}$ can be bounded by 
	\begin{align*}
		L_{\sigma_N^2}\leq 2N\max\limits_{
			\tilde{\bm{x}},\tilde{\bm{x}}'\in\mathbb{X}}k(\tilde{\bm{x}},
		\tilde{\bm{x}}')\frac{L_k}{\sigma_n^2}
	\end{align*}
	because $\|(\bm{K}(\bm{X}_N,\bm{X}_N)+\sigma_n^2\bm{I}_N)^{-1}\|\leq 
	\frac{1}{\sigma_n^2}$. Moreover, the unknown function admits a Lipschitz 
	constant $L_f$ with probability of at least $1-\delta_L/d$ with $\delta_L\in(0,1)$
	due to Corollary~\ref{th:Lip_f}. Therefore, application of the union bound 
	to \eqref{eq:un_err} yields
	\begin{align*}
	\gamma_N(\tau)\leq&\sqrt{2N\max\limits_{\tilde{\bm{x}},\tilde{\bm{x}}'
			\in\mathbb{X}}k(\tilde{\bm{x}},\tilde{\bm{x}}')\frac{L_k\tau(N)}{\sigma_n^2}}
	+L_f\tau(N)+L_k\frac{N\bar{f}+\sqrt{N(2\sqrt{N\eta_N}+2\eta_N+N)}}{\sigma_n^2}\tau(N)
	\end{align*}
	with probability of at least $1-\delta-\delta_f-\delta_L/d$. 
	This function must converge to $0$ for $N\rightarrow\infty$ in order to guarantee 
	a vanishing regression error which is only ensured if $\tau(N)$ decreases faster 
	than $\mathcal{O}(N^{-1})$. Therefore, set $\tau(N)\in\mathcal{O}(N^{-2})$ in order 
	to guarantee
	\begin{align*}
	\gamma_N(\tau_N)\in\mathcal{O}\left( N^{-1} \right).
	\end{align*}
	However, this choice of $\tau(N)$ implies that $\beta_N(\tau(N))\in\mathcal{O}(\log(N))$ 
	due to \eqref{eq:beta(tau)}. Since there exists an $\epsilon>0$ such that 
	$\sigma_N(\bm{x})\in\mathcal{O}\left( \frac{1}{\alpha(N)} \right)\subseteq \mathcal{O}\left(\log(N)^{-\frac{1}{2}-\epsilon}\right)$, 
	$\forall\bm{x}\in\mathbb{X}$ by assumption, we have
	\begin{align}
	\sqrt{\beta_N(\tau(N))}\sigma_N(\bm{x})+\gamma(\tau(N))\in\mathcal{O}\left( \log(N)^{-\epsilon} \right)
	\quad\forall\bm{x}\in\mathbb{X}
	\label{eq:ass_analysis}
	\end{align}
	with probability of at least $1-2\delta-\delta_L-\delta_f$ for all 
	$\delta\in(0,1)$, $\delta_f\in(0,1)$ and $\delta_L\in(0,1)$. It remains to show
	that this holds with probability $1$. We prove this by contradiction and consider
	$\delta=\delta_f=\delta_L$ for notational simplicity in the following. Assume that 
	\eqref{eq:ass_analysis} holds not with probability $1$. Then, there exists a 
	$\delta'\in(0,1)$ such that \eqref{eq:ass_analysis} holds with probability no more than 
	$1-\delta'$. However, we can set $\delta=\frac{1}{5}\delta'$ such that 
	\eqref{eq:ass_analysis} holds with probability of at least $1-\frac{4}{5}\delta'$ 
	which contradicts the assumption. Hence, \eqref{eq:ass_analysis} holds almost 
	surely and the proof is concluded.
\end{proof}

\begin{proof}[Proof of Corollary~\ref{th:as_err_data}]
	The corollary directly follows from Theorems~\ref{th:as_err} and \ref{th:varvan}.
\end{proof}

\section{Stability Proofs for Robust Tracking Control}

\begin{proof}[Proof of \cref{thm:Stable}]
	Since $\bm{\lambda}$ is Hurwitz, there exists a unique, positive definite matrix $\bm{P}\in\mathbb{R}^{d\times d}$ 
	such that \citep{Khalil2002}
	\begin{align*}
	\bm{A}^T\bm{P}+\bm{P}\bm{A}=-\bm{I}_d.
	\end{align*}
	Based on this matrix $\bm{P}$, we define a quadratic Lyapunov function~$V(\bm{e})=\bm{e}^T\bm{P}\bm{e}$. 
	The time derivative of this Lyapunov function is given by
	\begin{align*}
	\dot{V}(\bm{e})&=\bm{e}^T\bm{A}^T\bm{P}\bm{e}+\bm{e}^T\bm{P}\bm{A}\bm{e}+2\bm{e}^T\bm{p}_d(f(\bm{x})-\nu_N(\bm{x})),
	\end{align*}
	where $\bm{p}_d$ denotes the last column of the matrix $\bm{P}$.	Using the uniform error bound from \cref{th:errbound_with} 
	and the definition of $\bm{P}$, we can bound this derivative by
	\begin{align*}
	\dot{V}(\bm{e})&\leq -\|\bm{e}\|^2+2\left(\sqrt{\beta(\tau)}\sigma_N(\bm{x})+\gamma(\tau)\right)|\bm{e}^T\bm{p}_d|&\forall \bm{x}\in\mathbb{X}.
	\end{align*}
	Therefore, we obtain 
	\begin{align*}
	\dot{V}(\bm{x}-\bm{x}_{\mathrm{ref}})<0 &&\forall \bm{x}\in\mathbb{L}.
	\end{align*}
	If $\mathbb{X}\setminus\mathbb{L}\subseteq\mathbb{X}_0$, we can apply Lemma~\ref{lem:Lyap} to 
	characterize the ultimate bound by
	\begin{align*}
	\ubar{v}=\max\limits_{\bm{x}\in\mathbb{L}}(\bm{x}-\bm{x}_{\mathrm{ref}})^T\bm{P}(\bm{x}-\bm{x}_{\mathrm{ref}}).
	\end{align*}
\end{proof}

\begin{proof}[Proof of Corollary~\ref{th:asstab}]
	This corollary follows from a straight forward combination of Corollary~\ref{th:as_err_data} 
	and Theorem~\ref{thm:Stable}.
\end{proof}
\vskip 0.2in
\bibliography{example_paper}

\end{document}